\documentclass[11pt]{article} 

\usepackage{macros}

\begin{document}

\begin{center}

	{\bf{\LARGE{The Many Faces of Adversarial Risk}}}

	\vspace*{.25in}

 	\begin{tabular}{ccc}
 		{\large{Muni Sreenivas Pydi$^*$}} & \hspace*{.5in} & {\large{Varun Jog$^\dagger$}}\\
 		{\large{\texttt{pydi@wisc.edu}}} & \hspace*{.5in} & {\large{\texttt{vj270@cam.ac.uk}}} 
 			\end{tabular}
 \begin{center}
 Department of Electrical and Computer Engineering,\\University of Wisconsin-Madison$^*$\\
 Department of Pure Mathematics and Mathematical Statistics,\\University of Cambridge$^\dagger$
 \end{center}

	\vspace*{.2in}

January 2022

	\vspace*{.2in}

\end{center}

\begin{abstract}
Adversarial risk quantifies the performance of classifiers on adversarially perturbed data. Numerous definitions of adversarial risk---not all mathematically rigorous and differing subtly in the details---have appeared in the literature. In this paper, we revisit these definitions, make them rigorous, and critically examine their similarities and differences. Our technical tools derive from optimal transport, robust statistics, functional analysis, and game theory. Our contributions include the following: generalizing Strassen’s theorem to the unbalanced optimal transport setting with applications to adversarial classification with unequal priors; showing an equivalence between adversarial robustness and robust hypothesis testing with $\infty$-Wasserstein uncertainty sets; proving the existence of a pure Nash equilibrium in the two-player game between the adversary and the algorithm; and characterizing adversarial risk by the minimum Bayes error between a pair of distributions belonging to the $\infty$-Wasserstein uncertainty sets. Our results generalize and deepen recently discovered connections between optimal transport and adversarial robustness and reveal new connections to Choquet capacities and game theory.
\end{abstract}

\section{Introduction}

Neural networks are known to be vulnerable to \textit{adversarial attacks}, which are imperceptible perturbations to input data that maximize loss~\cite{SzeEtal13, GooEtal14, CarWag17}. Developing algorithms resistant to such attacks has received considerable attention in recent years \cite{CohEtal19, PapEtal16, MadEtal18, JalEtal17}, motivated by safety-critical applications such as autonomous driving \cite{GriEtal20, MuhEtal20}, medical imaging \cite{GreEtal16, MaEtal21, LiuEtal19} and law \cite{KumEtal18, ChaKam19}.

A classification algorithm with high accuracy (low risk) in the absence of an adversary may have poor accuracy (high risk) when an adversary is present. Thus, a modified notion known as \textit{adversarial risk} is used to quantify the adversarial robustness of algorithms. Algorithms that minimize adversarial risk are deemed robust. Procedures for finding them have been effective in practice~\cite{MadEtal18, UesEtal18, PapEtal16}, spurring numerous theoretical investigations into adversarial risk and its minimizers. 

There is no universally agreed upon definition of adversarial risk. Even the simplest setting of binary classification in $\real^d$ with an $\ell_2$ adversary admits various definitions involving set expansions \cite{DioEtal18, GouEtal19}, transport maps \cite{PinEtal20}, Markov kernels \cite{PydJog21}, and couplings \cite{MeuEtal21}. These works broadly interpret adversarial risk as a measure of robustness to small perturbations, but their definitions differ in subtle details such as the class of adversaries and algorithms considered, budget constraints placed on the adversary, assumptions on the loss function, and the geometries of decision boundaries.

\textit{Optimal adversarial risk} is most commonly defined as the minimax risk under adversarial contamination \cite{MadEtal18, ShaEtal18}. Other notable characterizations include an optimal transport cost between data generating distributions in~\cite{PydJog20, BhaEtal19, Doh19, Doh20}, 
the optimal value of a distributionally robust optimization problem \cite{StaJeg17, SinEtal17, TuEtal19}, and the value of
a two-player zero-sum game ~\cite{MeuEtal21, PinEtal20, BosEtal20, BulEtal16}.

The diversity of definitions for adversarial risk makes it challenging to compare approaches. Moreover, not all approaches are rigorous. For instance, the classes of adversarial strategies and classifier algorithms are often unclear, and issues of measurability are ignored. Although this may be harmless for applied research, it has led to incorrect proofs and insufficient assumptions in some theoretical works. A mathematically rigorous foundation for adversarial risk is essential for future research. 

In this paper, we examine various notions of adversarial risk in two settings: (1) binary classification in a non-parametric setting under $0$-$1$ loss function, where the decision boundary (or decision region) of a classifier is an arbitrary subset of the input space, and (2) multi-class classification in a parametric setting under a general loss function, where a classifier is parametrized by a $w$ in a hypothesis set $\cW$.  We present rigorous definitions of adversarial risk and identify conditions under which these definitions are equivalent. We consider the general setting of  Polish spaces (complete, separable metric spaces), and present stronger results for the Euclidean space ($\real^d$).
Our contributions are as follows:
\begin{itemize}
    \item \textbf{Well-definedness of adversarial risk: }
    We examine the definition of adversarial risk based on set expansions. For Polish spaces, we observe that adversarial risk is not Borel measurable, and hence, not well-defined when the decision region is an arbitrary Borel set (or, when the loss function is an arbitrary Borel measurable function). 
    We show that the problem can be resolved by considering a Polish space equipped with the universal completion of the Borel $\sigma$-algebra and restricting the decision regions to Borel sets (or by restricting the loss function to be upper semi-analytic, which is stronger than Borel measurability and weaker than universal measurability).
    For the Euclidean space with the Lebesgue $\sigma$-algebra, we show that adversarial risk is well-defined for any Lebesgue measurable decision region. Our key lemma (Lemma~\ref{lem: A_ep_porous}) shows that the Lebesgue $\sigma$-algebra is preferred over the Borel $\sigma$-algebra because set expansions are Lebesgue measurable but not necessarily Borel measurable. These results are contained in Section~\ref{sec: well-defined}. 
    
    \item \textbf{Equivalence between various notions of adversarial risk: } We show that the definition of adversarial risk using set expansions is identical to a notion of risk that appears in robust hypothesis testing with $\infty$-Wasserstein uncertainty sets. We prove this result in Polish spaces using the theory of measurable selections \cite{BerShr96, Wag77}. In $\real^d$, we are able to use the powerful theory of Choquet capacities \cite{Cho54} (in particular, Huber and Strassen's $2$-alternating capacities \cite{HubStr73}) to establish results of a similar nature. In addition, we derive the conditions under which this notion of adversarial risk is equivalent to alternative notions defined using transport maps and Markov kernels. These results are contained in Section~\ref{sec: W_infty}.
    
    \item  \textbf{Optimal transport characterization of optimal adversarial risk: } We consider the binary classification setup with unequal priors and show (under suitable assumptions) that the optimal adversarial risk from the above definitions is characterized by an unbalanced optimal transport cost between data-generating distributions. For both Polish spaces and $\real^d$, the main tool we use is Theorem~\ref{thm: generalized strassen} in which we generalize a classical result of Strassen on excess-cost optimal transport \cite{Str65, Vil03} from probability measures to finite measures with possibly unequal mass. This generalizes results of \cite{PydJog21, BhaEtal19} on binary classification, which were only for equal priors. These results are contained in Section~\ref{sec: strassen}.
    
    \item \textbf{Game-theoretic view on adversarial risk and existence of Nash equilibria: } We consider the setup of a zero-sum game between the adversary and the algorithm. We show that the value of this game (adversarial risk) is equal to the minimum Bayes error between a pair of distributions belonging to the $\infty$-Wasserstein uncertainty sets centered around true data-generating distributions. We prove the existence of a pure Nash equilibrium in this game for $\real^d$ and for Polish spaces with a \emph{midpoint property.} This extends/strengthens the results of \cite{MeuEtal21, PinEtal20, BosEtal20} to non-parametric classifiers. These results are contained in Section~\ref{sec: minimax}.
\end{itemize}

The paper is organized as follows: In Section~\ref{sec: preliminaries}, we present preliminary definitions from optimal transport and metric space topology. In Section~\ref{sec: defns}, we discuss various definitions of adversarial risk and present more related work. Sections~\ref{sec: well-defined},~\ref{sec: W_infty},~\ref{sec: strassen} and~\ref{sec: minimax} contain our main contributions summarized above. We conclude the paper in Section~\ref{sec: discussion} and discuss future research directions.

We emphasize that rectifying measure theoretic issues with existing formulations of adversarial risk is one of our contributions, but not the main focus of our paper. We start our presentation by addressing measurability and well-definedness (in Section~\ref{sec: well-defined}) because otherwise we will not be able to rigorously present our main results in the subsequent sections, namely: relation to robust hypothesis testing and Choquet capacities in Section~\ref{sec: W_infty}, generalizing the results of \cite{BhaEtal19, PydJog20} in Section, ~\ref{sec: strassen} proving minimax theorems and existence of Nash equilibria and extending the results of \cite{MeuEtal21, BosEtal20, PinEtal20} in Section~\ref{sec: minimax}.

\paragraph{Notation:} Throughout the paper, we use $\cX$ to denote a Polish space (a complete, separable metric space) with metric $d$ and Borel $\sigma$-algebra $\cB(\cX)$.
For $x\in \cX$ and $r\geq 0$, let $B_r(x)$ denote the closed ball of radius $r$ centered at $x$.
We use $\cP(\cX)$ and $\cM(\cX)$ to denote the space of probability measures and finite measures defined on the measure space $(\cX, \cB(\cX))$, respectively. 
Let $\overline{\cB}(\cX)$ denote the universal completion of $\cB(\cX)$. Let $\overline{\cP}(\cX)$ and $\overline{\cM}(\cX)$ denote the space of probability measures and finite measures defined on the complete measure space $(\cX, \overline{\cB}(\cX))$.
For $\mu, \nu\in \cM(\cX)$, we say $\nu$ \textit{dominates} $\mu$ if
$\mu(A)\leq \nu(A)$ for all $A\in \cB(\cX)$ and write $\mu \preceq\nu$.
When $\cX$ is $\real^d$, we use $\cL(\cX)$ to denote  the Lebesgue $\sigma$-algebra and $\lambda$ to denote the $d$-dimensional Lebesgue measure.
For a positive integer $n$, we use $[n]$ to denote the finite set $\{1, \ldots, n\}$.

\section{Preliminaries}\label{sec: preliminaries}

\subsection{Metric Space Topology}\label{sec: metric space topology}

We introduce three different notions of set expansions. For $\epsilon\geq 0$ and $A\in\cB(\cX)$, the \textit{$\epsilon$-Minkowski expansion} of $A$ is given by $A^{\oplus\epsilon} \defn \cup_{a\in A} B_\epsilon(a)$. The \textit{$\epsilon$-closed expansion} of $A$ is defined as $A^\epsilon \defn \{x\in \cX: d(x, A)\leq \epsilon\}$, where $d(x,A) = \inf_{a\in A} d(x,a)$. The \textit{$\epsilon$-open expansion} of $A$ is defined as $A^{\epsilon)} \defn \{x\in \cX: d(x, A)< \epsilon\}$. We use the notation $A^{-\epsilon}$ to denote $((A^c)^\epsilon)^c$. Similarly, $A^{\ominus\epsilon}\defn ((A^c)^{\oplus\epsilon})^c$. For example, consider the set $A = (0, 1]$ in the space $(\cX, d) = (\real, |\cdot|)$ and $\epsilon>0$. Then $A^{\oplus\epsilon} = (-\epsilon, 1+\epsilon]$, $A^\epsilon = [-\epsilon, 1+\epsilon]$ and $A^{\epsilon)} = (-\epsilon, 1+\epsilon)$. For any  $A\in \cB(\cX)$, $A^{\epsilon}$ is closed  and $A^{\epsilon)}$ is open. Hence, $A^{\epsilon}, A^{\epsilon)}\in \cB(\cX)$.
Moreover, $A^{\epsilon)} \subseteq A^{\oplus\epsilon} \subseteq A^{\epsilon}$.
However, $A^{\oplus\epsilon}$ may not be in $\cB(\cX)$ (see Lemma~\ref{lem: non-measurable}). In general, the Minkowski sum of two Borel sets need not be Borel~\cite{ErdSto70}, and that of two Lebesgue measurable sets need not be Lebesgue measurable~\cite{Sie20}.

\subsection{Optimal Transport}\label{sec: OT preliminaries}

Let $\mu, \nu \in \cP(\cX)$.
A \textit{coupling} between $\mu$ and $\nu$ is a joint probability measure $\pi\in \cP(\cX^2)$ with marginals $\mu$ and $\nu$. The set $\Pi(\mu, \nu)\subseteq\cP(\cX^2)$ denotes the set of all couplings between $\mu$ and $\nu$. 
The \textit{optimal transport cost} between $\mu$ and $\nu$ under a cost function   $c:\cX\times\cX\to [0, \infty)$ is defined as $\cT_c(\mu, \nu) = \inf_{\pi\in \Pi(\mu, \nu)} \int_{\cX^2} c(x,x')d\pi(x,x')$. For a positive integer $p$, the \textit{$p$-Wasserstein distance} between $\mu$ and $\nu$ is defined as, $W_p(\mu, \nu) = \left(\cT_{d^p}(\mu, \nu)\right)^\frac{1}{p}$.
The \textit{$\infty$-Wasserstein metric} is defined as $W_\infty(\mu, \nu) = \lim_{p\to\infty} W_p(\mu, \nu)$. It can also be expressed in the following ways~\cite{GivSho84}.
\begin{align}\label{eq: W_infty}
    W_\infty(\mu, \nu)
    =  \inf_{\pi\in \Pi(\mu, \nu)} \esssup _{(x,x')\sim \pi} d(x,x')
    = \inf\{\delta > 0:\mu(A) \leq \nu(A^\delta)\forall A\in \cB(\cX)\}.
\end{align}
Given a $\mu\in \cP(\cX)$ and a measurable function $f: \cX\to\cX$, the \textit{push-forward} of $\mu$ by $f$ is defined as a probability measure $f_{\sharp \mu}\in \cP(\cX)$ given by $f_{\sharp \mu} = \mu(f^{-1}(A))$ for all $A\in \cB(\cX)$.


\section{Adversarial Risk: Definitions and Related Work}\label{sec: defns}

In this section, we review several definitions for adversarial risk that are found in the literature. First, we consider a setting of general loss functions, where classifiers are parametrized by parameter $w$ in a hypothesis class $\cW$. Next, we consider a binary classification setting with the $0$-$1$ loss function, where non-parametric classifiers of the form $f_A(x) = \1\{x\in A\}$ correspond to decision regions $A\subseteq \cX$.

\subsection{General Loss Setting}

Let $\cX$ be the feature space, a Polish space equipped with a distance metric $d$. Let $\cY$ be a finite set of labels. Let $\rho$ be the true data distribution of labeled data points $(x,y)\in \cX\times\cY$, which can be expressed as $\rho(x,y) = \rho_y(y)\rho_{x|y}(x)$ where $\rho_y(y)$ is the marginal probability of label $y\in \cY$ and $\rho_{x|y}(x)$ is the conditional probability of $x\in \cX$ given the label $y$. Let $\cW$ denote the hypothesis class. Let $\ell:(\cX\times \cY)\times \cW\to [0, \infty]$ denote a loss function that is measurable with respect to $\cB(\cX)$ for all $w\in \cW$.

Consider a data-perturbing adversary of budget $\epsilon\geq 0$ that perturbs any data point $x\in \cX$ to $x'\in \cX$ such that $d(x,x')\leq \epsilon$. The adversarial risk of a classifier $w\in \cW$ under a loss function $\ell$ in the presence of such an adversary is given by, 
\begin{align}\label{eq: adv risk exp sup}
    R_{\oplus\epsilon}(\ell, w) = \E_{(x,y)\sim \rho}\left[ \sup_{d(x,x')\leq \epsilon} \ell((x', y),w) \right].
\end{align}

If the loss function $\ell(\cdot, w)$ is  upper semi-continuous and bounded above for all $w\in \cW$, Meunier et al. (\cite{MeuEtal21}) show that $R_{\oplus\epsilon}(\ell, w)$ is well-defined. But in general, it may not be so. 

One way to resolve measurability issues is to restrict the adversary to use measurable transport maps for data perturbation.
Let $F \defn \{f_y:\cX \to \cX, f_y \text{ is } \rho_y-\text{measurable} \  | \  y\in \cY\}$ denote a collection of measurable maps for each label $y\in \cY$. We say that $F$ is of budget $\epsilon$ (denoted by $F\in F_\epsilon$) if $d(x, f_y(x))\leq \epsilon$ with probability $1$ for $(x, y)\sim \rho$. Under such an adversary, the adversarial risk may be defined as follows.
\begin{align}\label{eq: R_F gen loss}
    R_{F_\epsilon}(\ell, w) = \sup_{F\in F_\epsilon}\E_{(x,y)\sim \rho}\left[  \ell((f_y(x),y),w) \right].
\end{align}
The above definition was used for the binary classification setting in \cite{PinEtal20}. A more general definition for adversarial risk was proposed in \cite{PydJog21} using Markov kernels. Let $\kappa$ denote a set of Markov kernels $\kappa_y$ for $y\in \cY$.
Let $\rho^\kappa_{(x,y,x')}$ denote the joint distribution of $(x,y,x')$ induced by $\kappa$.
We say that the Markov kernel adversary $\kappa$ has a budget $\epsilon$ (denoted by  $\kappa\in K_\epsilon$) if $d(x,x')\leq\epsilon$, $\rho^\kappa_{(x,x')|y}$-a.s. where $\rho^\kappa_{(x,x')|y}\in \cP(\cX\times\cX)$ denotes the conditional distribution of $(x,x')$ given $y\in \cY$ and $x'$ is the perturbation of the data point $x$ with label $y$ using the Markov kernel $\kappa_y\in \kappa$. Under such a Markov kernel adversary, adversarial risk is defined as the following in \cite{PydJog21}.
\begin{align}
    R_{K_\epsilon}(\ell, w) = \sup_{\kappa\in K_\epsilon}\E_{(x,y,x')\sim \rho^\kappa_{(x,y,x')}}\left[  \ell((x', y),w) \right].
\end{align}

Another way to define adversarial risk is by considering perturbations to the input data distributions rather than individual data points. Optimal transport-based perturbations, in particular the $\infty$-Wasserstein metric (denoted by $W_\infty$) has been used to define such perturbations (\cite{PydJog21, MeuEtal21}). Let an adversary $\gamma$ be defined as a collection of perturbed probability distributions for each label i.e., $\gamma\defn \{\rho^\gamma_{x'|y}\in \cP(\cX) | y\in \cY\}$. We say that the adversary $\gamma$ has a budget $\epsilon$ (denoted by $\Gamma_\epsilon$) if $W_\infty(\rho_{x|y}, \rho^\gamma_{x'|y})\leq\epsilon$ for all $y\in \cY$. Under such a distribution perturbing adversary, the adversarial risk is defined as,
\begin{align}
    R_{\Gamma_\epsilon}(\ell, w) = \sup_{\gamma\in \Gamma_\epsilon}\E_{(x',y)\sim \rho_y\rho^\gamma_{x'|y}}\left[  \ell((x', y),w) \right].
\end{align}
The use of $\infty$-Wasserstein metric for defining adversarial risk is motivated by the following fact: For $\mu, \nu\in \cP(\cX)$, $W_\infty(\mu, \nu)\leq \epsilon$ if and only if there exists a coupling (a joint probability distribution) $\pi\in \Pi(\mu, \nu)$ such that $d(x,x')\leq \epsilon$ with probability $1$ for $(x,x')\sim \pi$. That means, all the probability mass under the distribution $\mu$ may be transported to $\nu$ without transporting any mass by more than $\epsilon$ almost surely.

The following inequality is an immediate consequence of the above definitions of adversarial risk:
\begin{align}\label{eq: adv risk gen loss inequality}
R_{F_\epsilon}(\ell, w) \leq R_{K_\epsilon}(\ell, w)\leq R_{\Gamma_\epsilon}(\ell, w).
\end{align}
We shall investigate conditions for equality in the above inequality and relations between the above three formulations of adversarial risk and the classical formulation $R_{\oplus\epsilon}(\ell, w)$.

\subsection{Binary Classification with $0$-$1$ Loss Setting}

In this subsection, we consider a binary classification setting where the label space $\cY=\{0,1\}$.  Let $p_0, p_1\in \cP(\cX)$ be the data-generating distributions for labels $0$ and $1$, respectively. Let the prior probabilities for labels $0$ and $1$ be in the ratio $T:1$ where we assume $T\geq 1$ without loss of generality. 
For any set $A\in \cB(\cX)$, we may consider a classifier $f_A(x)\defn \1\{x\in A\}$ which labels any point in the set $A$ as $1$ and any point in $A^c$ as $0$. We say that such a classifier has a decision region $A$.
The error (standard risk) incurred by such a classifier under the $0$-$1$ loss function is,
$R_{\oplus 0}(\ell_{0/1}, A) = \frac{T}{T+1} p_0(A) + \frac{1}{T+1} p_1(A^c)$.

An adversary of budget $\epsilon>0$ can perturb any $x\in \cX$ to $x'\in B_\epsilon(x)$. It follows that any $x\in A$ can be perturbed to $x'\in \cup_{a\in A} B_\epsilon(a) =  A^{\oplus\epsilon}$. Hence, adversarial risk could be defined as
\begin{align}\label{eq: adv risk minkowski set expansion}
    R_{\oplus\epsilon}(\ell_{0/1}, A) = \frac{T}{T+1} p_0(A^{\oplus\epsilon}) + \frac{1}{T+1} p_1((A^c)^{\oplus\epsilon}).
\end{align}
The above formulation is a special case of \eqref{eq: adv risk exp sup} for the $0$-$1$ loss function. Indeed, for $x\in \cX$ and $y\in \{0,1\}$, $\ell_{0/1} ((x,y), A) = \1\{x\in A, y=0\} + \1\{x\in A^c, y=1\}$. Hence,
\begin{align*}
    &R_{\oplus\epsilon}(\ell_{0/1}, A)\\
    &= \frac{T}{T+1} \E_{p_0} \left[ \sup_{d(x,x')\leq \epsilon} \1\{x'\in A\} \right] +
    \frac{1}{T+1} \E_{p_1} \left[ \sup_{d(x,x')\leq \epsilon} \1\{x'\in A^c\} \right]\\
    &= \frac{T}{T+1} p_0(A^{\oplus\epsilon}) + \frac{1}{T+1} p_1((A^c)^{\oplus\epsilon}).
\end{align*}

A problem with the formulation in equation~\eqref{eq: adv risk minkowski set expansion} is the ambiguity over the measurability of the sets $A^{\oplus\epsilon}$ and $(A^c)^{\oplus\epsilon}$. 
Even when $A\in \cB(\cX)$, it is not guaranteed that $A^{\oplus\epsilon}, (A^c)^{\oplus\epsilon}\in \cB(\cX)$ (see Appendix~\ref{app: well-defined 0-1 loss} for an example). Hence, $R_{\oplus\epsilon}(\ell_{0/1}, A)$ is not well-defined for all $A\in \cB(\cX)$. It is shown in \cite{PydJog21} that $R_{\oplus\epsilon}(\ell_{0/1}, A)$ is well-defined when $A$ is either closed or open. A simple fix to this measurability problem is to use closed set expansion instead of the Minkowski set expansion, as done in \cite{MahEtal19}. This leads to the following formulation.
\begin{align}\label{eq: adv risk closed set expansion}
    R_\epsilon(\ell_{0/1}, A) = \frac{T}{T+1} p_0(A^\epsilon) + \frac{1}{T+1} p_1((A^c)^\epsilon).
\end{align}
The above definition is well-defined for any $A\in \cB(\cX)$ because $A^\epsilon$ and $(A^c)^\epsilon$ are both closed and hence, measurable. However, under the above definition, a point $x\in A$ may be perturbed to $x'\in A^\epsilon$ such that $d(x,x')>\epsilon$. For example, when $A = (0,1)$, we have $A^{\epsilon} = [-\epsilon, \epsilon]$ and an adversary may transport $x = \delta>0$ to $x'=-\epsilon$, violating the budget constraint at $x$. 

\begin{remark*}
The formulations in equations \eqref{eq: adv risk exp sup}, \eqref{eq: adv risk minkowski set expansion} and \eqref{eq: adv risk closed set expansion}  can give a strictly positive adversarial risk even for a ``perfect'' (i.e., Bayes optimal) classifier. 
This is consistent with the literature on adversarial examples where even a perfect classifier is forced to make errors in the presence of evasion attacks. 
These formulations of adversarial risk correspond to ``constant-in-the-ball'' risk of \cite{GouEtal19} and ``corrupted-instance'' risk in \cite{DioEtal18, MahEtal19}. Here, an adversarial risk of zero is only possible if the supports of $p_0$ and $p_1$ are non-overlapping and separated by at least $2\epsilon$. This is not the case with other formulations of adversarial risk such as ``exact-in-the-ball'' risk~\cite{GouEtal19}, ``prediction-change risk and ``error-region'' risk~\cite{DioEtal18, MahEtal19}. We focus on the ``corrupted-instance'' family of risks in this work.
\end{remark*}

Another approach to defining adversarial risk is by explicitly defining the class of adversaries of budget $\epsilon$ as measurable transport maps $f:\cX\to \cX$ that push-forward the true data distribution such that no point is transported by more than a distance of $\epsilon$; i.e., $d(x, f(x))\leq\epsilon$.  The transport map-based adversarial risk \cite{PinEtal20} is formally defined as follows:
\begin{align}\label{eq: transport map adv risk}
    R_{F_\epsilon}(\ell_{0/1}, A) 
    &= \sup_{\substack{f_0, f_1:\cX\to\cX\\  \forall x\in \cX, d(x,f_i(x))\leq \epsilon}}
    \frac{T}{T+1} f_{0\sharp p_0}(A) + \frac{1}{T+1} f_{1\sharp p_1}((A^c)).
\end{align}
It is easy to see that the above definition is a special case of the definition in equation~\eqref{eq: R_F gen loss} for the $0$-$1$ loss function.
Yet another definition uses the robust hypothesis framework with $W_\infty$ uncertainty sets. In this approach, an adversary perturbs the true distribution $p_i$ to a corrupted distribution $p_i'$ such that $W_\infty(p_i, p_i')\leq\epsilon$. From \eqref{eq: W_infty}, this is equivalent to the existence of a coupling $\pi\in \Pi(p_i, p_i')$ such that $\esssup_{(x,x')\sim \pi}d(x,x')\leq \epsilon$. The adversarial risk with such an adversary is given by 
\begin{align}\label{eq: w_infty optimal adv risk}
    R_{\Gamma_\epsilon}(\ell_{0/1}, A) 
    &= \sup_{W_\infty(p_1, p_1'), W_\infty(p_0, p_0')\leq\epsilon} \frac{T}{T+1} p_0'(A) + \frac{1}{T+1} p_1'((A^c)).    
\end{align}
Clearly, $R_{F_\epsilon}(\ell_{0/1}, A) \leq R_{\Gamma_\epsilon}(\ell_{0/1}, A)$, but conditions for equality were not studied in prior work. Moreover, their relation to set expansion-based definitions in \eqref{eq: adv risk minkowski set expansion} and \eqref{eq: adv risk closed set expansion} was also unknown. 

Next we discuss some characterizations of optimal adversarial risk, defined as,
\begin{align}\label{eq: optimal adv risk}
    R^*_{\oplus\epsilon} \defn \inf_{A\in \cB(\cX)} R_{\oplus\epsilon}(\ell_{0/1}, A).
\end{align}
In \cite{PydJog20, BhaEtal19}, it is shown that $R^*_{\oplus\epsilon}  = \frac{1}{2}[1-D_\epsilon(p_0, p_1)]$ for equal priors ($T=1$), where $D_\epsilon$ is an optimal transport cost defined as follows.
\begin{definition}[$D_\epsilon$ cost]
Let $\mu, \nu \in \cP(\cX)$. Let $\epsilon\geq 0$.
Let $c_\epsilon:\cX^2\to \{0,1\}$ be such that 
$c_\epsilon(x,x') = \1\{(x,x')\in \cX\times\cX :    d(x,x')>2\epsilon\}$.
Then for $\mu, \nu \in \cP(\cX)$ and $\epsilon\geq 0$, $D_\epsilon(\mu, \nu) = \cT_{c_\epsilon}(\mu, \nu)$.
\end{definition}
For $\epsilon = 0$, $D_\epsilon$ reduces to the total variation distance. While $D_0$ is a metric on $\cP(\cX)$, $D_\epsilon$ (for $\epsilon>0$) is neither a metric nor a pseudometric \cite{PydJog21}. Other formulations of optimal adversarial risk are inspired from game theory \cite{PinEtal20, MeuEtal21, BosEtal20}. 
Consider a game between two players: (1) The adversary whose action space is pairs of distributions $p_0', p_1' \in \overline{\cP}(\cX)$, and (2) The algorithm whose action space is the space of decision regions of the form $A\in \cB(\cX)\}$. For $T>0$, define the payoff function, $r: \cB(\cX) \times \overline{\cP}(\cX)\times\overline{\cP}(\cX) \to [0,1]$ as,
\begin{align*}
    r(A, \mu, \nu) = \frac{T}{T+1} \mu(A) + \frac{1}{T+1} \nu((A^c)).
\end{align*}
The payoff when the algorithm plays first is given by $\inf_{A\in \cB(\cX)} 
    \sup_{W_\infty(p_0, p_0'), W_\infty(p_1, p_1')\leq \epsilon}  r(A, p_0', p_1')$, and this quantity is interpreted as the optimal adversarial risk in this setup.

\section{Well-Definedness of Adversarial Risk}\label{sec: well-defined}

In this section, we discuss the conditions under which the definitions for adversarial risk presented in Section~\ref{sec: defns} are well-defined. In Subsection~\ref{sec: well-defined 0-1 loss} we present the results for the binary classification setting under $0$-$1$ loss and in Subsection~\ref{sec: well-defined gen loss} we discuss the setting of more general loss functions.

\subsection{Binary Classification with $0$-$1$ Loss Setting}\label{sec: well-defined 0-1 loss}
As stated in Section~\ref{sec: defns}, $R_{\oplus\epsilon}(\ell_{0/1}, A)$ may not be well-defined for some decision regions $A\in\cB(\cX)$ because of the non-measurability of the sets $A^{\oplus\epsilon}$ and $(A^c)^{\oplus\epsilon}$. Specifically, we have the following lemma. 

\begin{lemma}\label{lem: non-measurable}
For any $\epsilon>0$, there exists $A\in \cB(\cX)$ such that $A^{\oplus\epsilon}\notin \cB(\cX)$.
\end{lemma}

The proof of Lemma~\ref{lem: non-measurable} is in Appendix~\ref{app: well-defined 0-1 loss}.

In this section, we lay down the conditions under which the ambiguity on the measurability of $A^{\oplus\epsilon}$ can be resolved.
We begin by presenting a lemma that shows that $A^{\oplus\epsilon}$ is an analytic set (i.e., a continuous image of a Borel set) whenever $A$ is Borel. It is known that analytic sets are universally measurable; i.e., they belong in $\overline{\cB}(\cX)$, the universal completion of the Borel $\sigma$-algebra $\cB(\cX)$, and are measurable with respect to any finite measure defined on the complete measure space, $(\cX, \overline{\cB}(\cX))$.

\begin{lemma}\label{lem: Borel set expansion}
Let $A\in \cB(\cX)$. Then, $A^{\oplus\epsilon}$ is an analytic set. Consequently, $A^{\oplus\epsilon}\in \overline{\cB}(\cX)$.
\end{lemma}

The proof of Lemma~\ref{lem: Borel set expansion} is in Appendix~\ref{app: well-defined 0-1 loss}. By virtue of Lemma~\ref{lem: Borel set expansion}, we have the following.

\begin{theorem}\label{thm: Adv risk in complete measurable space well-defined}
Let $p_0, p_1\in \overline{\cP}(\cX)$. Then for any  $A\in \cB(\cX)$,  $R_{\oplus\epsilon}(\ell_{0/1}, A)$ is well-defined. 
\end{theorem}

The proof of Theorem~\ref{thm: Adv risk in complete measurable space well-defined} is in Appendix~\ref{app: well-defined 0-1 loss}.
For the special case of $\cX = \mathbb{R}^d$, we can further strengthen Theorem~\ref{thm: Adv risk in complete measurable space well-defined} to include all Lebesgue measurable sets $\cL(\cX)$ instead of just Borel sets $\cB(\cX)$. For this, we use the concept of porous sets.

\begin{definition}[Porous set]
A set $E\subseteq \cX$ is said to be porous if there exists $\alpha\in (0,1)$ and $r_0>0$ such that for every $r\in (0, r_0]$ and every $x\in \cX$, there is an $x'\in \cX$ such that $B_{\alpha r}(x')\subseteq B_r(x)\backslash E$.
\end{definition}
Porous sets are a subclass of nowhere dense sets. Importantly, $\lambda(E) = 0$ for any porous set $E\subseteq \mathbb{R}^d$ \cite{Zaj05}. By the following lemma, the set difference between the closed/open set expansions is porous.
\begin{lemma}\label{lem: A_ep_porous}
Let $(\cX,d) =(\mathbb{R}^d, \|\cdot\|)$ and $A\in \cL(\cX)$. Then $E = A^{\epsilon}\backslash A^{\epsilon)}$ is porous.
\end{lemma}

The proof of Lemma~\ref{lem: A_ep_porous} is in Appendix~\ref{app: well-defined 0-1 loss}.
Lemma~\ref{lem: A_ep_porous} plays a crucial role in proving that $A^{\oplus\epsilon}\in \cL(\cX)$ whenever $A\in \cL(\cX)$. We recall that $A^{\oplus\epsilon}$ is the Minkowski sum of $A$  with the closed  $\epsilon$-ball. In general, the Minkowski sum of two Lebesgue measurable sets is not always Lebesgue measurable \cite{Sie20, Gar02}. So the fact that one of them is a closed ball in case of $A^{\oplus\epsilon}$ is important.
In the following theorem, we use Lemma~\ref{lem: A_ep_porous} to prove the measurability of $A^{\oplus\epsilon}$ and in turn prove that $R_{\oplus\epsilon}(\ell_{0/1}, A)$ is well-defined for any $A\in \cL(\cX)$.

\begin{theorem}\label{thm: Adv risk in Rd well-defined}
Let $(\cX,d) =(\mathbb{R}^d, \|\cdot\|)$.
Let $p_0, p_1\in \overline{\cP}(\cX)$ and let $\epsilon\geq 0$.
Then for any $A\in \cL(\cX)$, $R_{\oplus\epsilon}(\ell_{0/1}, A)$ is well-defined.
If, in addition, $p_0$ and $p_1$ are absolutely continuous with respect to the Lebesgue measure, then $R_{\oplus\epsilon}(\ell_{0/1}, A) = R_{\epsilon}(\ell_{0/1}, A)$.
\end{theorem}
\begin{proof}
By Lemma~\ref{lem: A_ep_porous} $A^{\epsilon}\backslash A^{\epsilon)}$ is porous, and so $\lambda(A^{\epsilon}\backslash A^{\epsilon)}) = 0$.
Hence, 
$\lambda(A^{\epsilon}) = \lambda(A^{\epsilon)})$.
Using the fact that  $A^{\epsilon)} \subseteq A^{\oplus\epsilon} \subseteq A^{\epsilon}$, we have $A^{\oplus\epsilon}\backslash A^{\epsilon)} \subseteq A^{\epsilon}\backslash A^{\epsilon)}$. Hence, $\lambda(A^{\oplus\epsilon}\backslash A^{\epsilon)}) = 0$. 
Therefore, $A^{\oplus\epsilon}\in \cL(\cX)$ and  $\lambda(A^{\oplus\epsilon}) = \lambda(A^{\epsilon}) = \lambda(A^{\epsilon)})$.

Since $A^{\oplus\epsilon}, (A^c)^{\oplus\epsilon}\in \cL(\cX)$, $R_{\oplus\epsilon}(\ell_{0/1}, A)$ is well-defined. 
If $p_0$ and $p_1$ are absolutely continuous with respect to the Lebesgue measure, the equation $R_{\oplus\epsilon}(\ell_{0/1}, A) = R_{\epsilon}(\ell_{0/1}, A)$ follows from   the previous conclusion that $\lambda(A^{\oplus\epsilon}) = \lambda(A^{\epsilon})$.
\end{proof}

\subsection{General Loss Setting}\label{sec: well-defined gen loss}

In the expected-supremum formulation of adversarial risk shown in \eqref{eq: adv risk exp sup}, the worst-case loss function $\sup_{d(x,x')\leq \epsilon} \ell((x', y),w)$ may not be measurable even when $\ell((x', y),w)$ is measurable for every $x'\in \cX$ because the supremum is taken over an uncountable family of measurable functions. In this subsection, we resolve this ambiguity over the measurability of the worst-case loss function.

A real-valued function $\phi:\cX\to \real$ is called upper semi-analytic if the set $\{x\in \cX: \phi(x)>t\}$ is an analytic set for every $t\in \real$. Since every Borel set is an analytic set, it follows that every Borel measurable function is upper semi-analytic. However, the converse is not true in general. Nevertheless, upper semi-analytic functions are universally measurable owing to the fact that analytic sets are universally measurable. We now present a lemma that shows that the worst-case loss function $\sup_{d(x,x')\leq \epsilon} \ell((x', y),w)$ is universally measurable if $\ell((\cdot, y), w)$ is upper semi-analytic for all $y\in \cY$ and $w\in \cW$.

\begin{lemma}\label{lem: gen loss function measurable}
If the loss function $\ell((\cdot, y), w)$ is upper semi-analytic for all $y\in \cY$ and $w\in \cW$, then  the worst-case loss function $\sup_{d(x,x')\leq \epsilon} \ell((x', y),w)$ is also upper semi-analytic and hence universally measurable. Therefore, $R_{\oplus\epsilon}(\ell, w)$ is well-defined on the measure space $(\cX, \overline{\cB}(\cX))$.
\end{lemma}

The proof of Lemma~\ref{lem: gen loss function measurable} is in Appendix~\ref{app: well-defined gen loss}.
For the special case of $\cX = \real^d$, we can further extend the measurability of the worst-case loss function from upper semi-analytic functions to the more general Lebesgue measurable functions, as shown in the following lemma.

\begin{lemma}\label{lem: gen loss function measurable in R^d}
Let $(\cX,d) =(\mathbb{R}^d, \|\cdot\|)$. Then, $R_{\oplus\epsilon}(\ell, w)$ is well-defined for any loss function $\ell:(\cX\times \cY)\times \cW\to [0, \infty]$ for which $\ell((\cdot, y), w)$ is Lebesgue measurable for all $y\in \cY$ and $w\in \cW$.
\end{lemma} 

The proof of Lemma~\ref{lem: gen loss function measurable in R^d} is in Appendix~\ref{app: well-defined gen loss}.

Now that we have established the conditions for which $R_{\oplus\epsilon}$ is well-defined, in the next section, we explore its relation to other notions of adversarial risk.

\section{Equivalence with $\infty$-Wasserstein Robustness}\label{sec: W_infty}

In this section, we show the conditions under which $R_{\oplus\epsilon}(\ell_{0/1}, A)$ is equivalent to other notions of adversarial risk based on transport maps and $W_\infty$ robustness. The equivalences established in this section are summarized in Tables~\ref{table: equivalences 0-1 loss} and \ref{table: equivalences gen loss}. In Subsection~\ref{sec: measurable selection}, we consider general Polish spaces and in Subsection~\ref{sec: capacities}, we consider the Euclidean space.

\begin{table}[!t]
\begin{center}
\caption{Equivalences among adversarial risk formulations for $0$-$1$ loss. $R_{\oplus\epsilon}(A), R_\epsilon(A), R_{F_\epsilon}(A)$ and $R_{\Gamma_\epsilon}(A)$ denote adversarial risk  for $0$-$1$ loss function ($\ell_{0/1}$) for a binary classifier with decision region $A$ (i.e. $f_A(x) = \1\{x\in A\}$), defined using Minkowski set expansions, closed set expansions, transport maps and $\infty$-Wasserstein metric respectively. $\cB(\cX)$ and $\cL(\cX)$ denote the Borel and Lebesgue $\sigma$-algebras. $(\cX, \overline{\cB}(\cX))$ denotes the universal completion of the Borel measure space, $(\cX, {\cB}(\cX))$. }
\label{table: equivalences 0-1 loss}
\vspace{3mm}
\begin{tabular}{|l|l|}
\hline
\multicolumn{1}{|c|}{Equivalences in Adversarial Risk}                                 & \multicolumn{1}{c|}{Conditions}                                             \\ \hline\hline
$R_{\oplus\epsilon}(A) = R_{\Gamma_\epsilon}(A)$                                       & $\real^d$: $A\in \cL(\cX)$ or $(\cX, \overline{\cB}(\cX))$: $A\in \cB(\cX)$ \\
$R_{\oplus\epsilon}(A)  = R_{\Gamma_\epsilon}(A)= R_{F_\epsilon}(A)$                   & $(\cX, \overline{\cB}(\cX))$: $A\in \cB(\cX)$                               \\
$R_{\oplus\epsilon}(A) = R_{\Gamma_\epsilon}(A)= R_{F_\epsilon}(A) = R_\epsilon(A)$ & $\real^d$: $A\in \cL(\cX)$ and $p_0, p_1$ have densities                    \\ \hline
\end{tabular}
\end{center}
\end{table}

\begin{table}[!t]
\begin{center}
\caption{Equivalences among adversarial risk formulations for general loss. $R_{\oplus\epsilon}(w),  R_{F_\epsilon}(w), R_{K_\epsilon}(w)$ and $R_{\Gamma_\epsilon}(w)$ denote adversarial risk  for a loss function $\ell$ for a classifier parametrized by $w\in \cW$, defined using expected supremum,  transport maps, Markov kernels and $\infty$-Wasserstein metric respectively.  $(\cX, \overline{\cB}(\cX))$ denotes the universal completion of the Borel measure space. }
\label{table: equivalences gen loss}
\vspace{3mm}
\begin{tabular}{|l|l|}
\hline
\multicolumn{1}{|c|}{Equivalences in Adversarial Risk}                              & \multicolumn{1}{c|}{Conditions}                                                                                                                                                                                                   \\ \hline
$R_{\oplus\epsilon}(w) = R_{\Gamma_\epsilon}(w)$                                    & \begin{tabular}[c]{@{}l@{}}$\real^d$: $\ell((x, y), w)$ Lebesgue measurable  in $x$, or\\ $(\cX, \overline{\cB}(\cX))$: $\ell((x, y), w)$ upper semi-analytic in $x$\end{tabular} \\
$R_{\oplus\epsilon}(w) = R_{\Gamma_\epsilon}(w)= R_{F_\epsilon}(w) = R_{K_\epsilon}(w)$ &  $(\cX, \overline{\cB}(\cX))$: $\ell((x, y), w)$ upper semi-continuous in $x$                                                                                                                                                                         \\ \hline
\end{tabular}
\end{center}
\end{table}

\subsection{$W_\infty$ Robustness in Polish Spaces via  Measurable Selections}\label{sec: measurable selection}

We begin by presenting a lemma that links the measure of $\epsilon$-Minkowsi set expansion to the worst case measure over a $W_\infty$ probability ball of radius $\epsilon$.
\begin{lemma}\label{lem: sup measure for Borel sets}
Let $\mu\in \overline{\cP}(\cX)$ and $A\in \cB(\cX)$. Then $\sup_{W_\infty(\mu, \mu')\leq\epsilon} \mu'(A) =\mu(A^{\oplus\epsilon})$.
Moreover, the supremum in the previous equation is achieved by a $\mu^*\in \cP(\cX)$ that is induced from $\mu$ via a measurable transport map $\phi:\cX\to\cX$ (i.e. $\mu^* = \phi_{\sharp \mu}$) satisfying $d(x,\phi(x))\leq \epsilon$ for all $x\in\cX$.
\end{lemma}

The proof of Lemma~\ref{lem: sup measure for Borel sets} is in Appendix~\ref{app: measurable selection}.
A crucial step in the proof of Lemma~\ref{lem: sup measure for Borel sets} is finding a measurable transport map $\phi$ such that $\phi^{-1}(A) = A^{\oplus\epsilon}$ and $d(x,\phi(x))\leq \epsilon$ for all $x\in\cX$.
 In the following theorem, we use Lemma~\ref{lem: sup measure for Borel sets} to establish the equivalence between three different notions of adversarial risk introduced in section~\ref{sec: defns}.
\begin{theorem}\label{thm: Adv risk in Polish Space}
 Let $p_0, p_1\in \overline{\cP}(\cX)$ and $A\in \cB(\cX)$. Then $R_{\oplus\epsilon}(\ell_{0/1}, A)
    = R_{F_\epsilon}(\ell_{0/1}, A) 
    = R_{\Gamma_\epsilon}(\ell_{0/1}, A)$. In addition, the supremum over $f_0$ and $f_1$ in $R_{F_\epsilon}(\ell_{0/1}, A)$ is attained. Similarly, the supremum over $p_0'$ and $p_1'$ in $R_{\Gamma_\epsilon}(\ell_{0/1}, A)$ is attained.
\end{theorem}

\begin{proof}
Since $A\in \cB(\cX)$, $A^c\in \cB(\cX)$ and by Lemma~\ref{lem: Borel set expansion}, $A^{\oplus\epsilon}, (A^c)^{\oplus\epsilon} \in \overline{\cB}(\cX)$.
Therefore $R_{\oplus\epsilon}(\ell_{0/1}, A)$ is well-defined. By Lemma~\ref{lem: sup measure for Borel sets}, we have
\begin{align*}
    R_{\Gamma_\epsilon}(\ell_{0/1}, A) 
    &= \sup_{\substack{W_\infty(p_0, p_0')\leq\epsilon\\W_\infty(p_1, p_1')\leq\epsilon}} \frac{T}{T+1} p_0'(A) + \frac{1}{T+1} p_1'((A^c))\\
    &= \frac{T}{T+1} \left( \sup_{W_\infty(p_0, p_0')\leq\epsilon}  p_0'(A)\right) + \frac{1}{T+1} \left( \sup_{W_\infty(p_1, p_1')\leq\epsilon}  p_1'((A^c))\right)\\
    &= \frac{T}{T+1} p_0(A^{\oplus\epsilon}) + \frac{1}{T+1} p_1((A^c)^{\oplus\epsilon})\\
    &= R_{\oplus\epsilon}(\ell_{0/1}, A).
\end{align*}
By Lemma~\ref{lem: sup measure for Borel sets} again, the supremum over $p_0'$ and $p_1'$ in $R_{\Gamma_\epsilon}(\ell_{0/1}, A)$ is attained by measures pushed forward from $p_0$ and $p_1$ via some measurable maps $f_0$ and $f_1$. From this, the remaining assertions of the theorem follow.
\end{proof}

We will now extend the above result to more general loss functions. The following lemma plays a critical role in doing this.

\begin{lemma}\label{lem: sup expectation for Polish}
Let $\mu\in \overline{\cP}(\cX)$. Then for any  real-valued upper semi-analytic function function $\phi: \cX\to [0, \infty)$, 
\begin{align}\label{eq: sup expectation for Polish}
    \sup_{W_\infty(\mu, \mu')\leq\epsilon} \E_{x\sim \mu'}[\phi(x)] = \E_{x\sim \mu} \left[ \sup_{d(x,x')\leq \epsilon} \phi(x') \right].
\end{align}
Moreover, if the function $\phi$ is upper semi-continuous, then the supremum on the left hand side in the previous equation is achieved by a $\mu^*\in \overline{\cP}(\cX)$ that is induced from $\mu$ via a universally measurable transport map $m:\cX\to\cX$ (i.e. $\mu^* = m_{\sharp \mu}$) satisfying $d(x,m(x))\leq \epsilon$ for all $x\in\cX$.
\end{lemma}

The proof of Lemma~\ref{lem: sup expectation for Polish} is in Appendix~\ref{app: measurable selection}.
Using Lemma~\ref{lem: sup expectation for Polish}, we prove the following theorem, which generalizes Theorem~\ref{thm: Adv risk in Polish Space} to more general loss functions.

\begin{theorem}\label{thm: Adv risk in Polish gen loss}
If the loss function $\ell((\cdot, y), w)$ is upper semi-analytic for all $y\in \cY$ and $w\in \cW$, then 
$R_{\oplus\epsilon}(\ell, w) = R_{\Gamma_\epsilon}(\ell, w)$.
If in addition, $\ell((\cdot, y), w)$ is upper semi-continuous for all $y\in \cY$ and $w\in \cW$, then 
$R_{\oplus\epsilon}(\ell, w) =R_{F_\epsilon}(\ell, w) = R_{K_\epsilon}(\ell, w) =  R_{\Gamma_\epsilon}(\ell, w)$.
\end{theorem}
\begin{proof}

\begin{align*}
    R_{\Gamma_\epsilon}(\ell, w)
    &= \sup_{\gamma\in \Gamma_\epsilon}\E_{(x',y)\sim \rho_y\rho^\gamma_{x'|y}}\left[  \ell((x', y),w) \right]\\
    &= \E_{(x,y)\sim \rho_y\rho_{x|y}} \left[ \sup_{d(x,x')\leq \epsilon} \ell((x', y),w) \right]\\
    &= R_{\oplus\epsilon}(\ell, w),
\end{align*}
where the second inequality follows from Lemma~\ref{lem: sup expectation for Polish} because of the assumption that $\ell((\cdot, y), w)$ is upper semi-analytic for all $y\in \cY$ and $w\in \cW$.

With the stronger assumption that $\ell((\cdot, y), w)$ is upper semi-continuous for all $y\in \cY$ and $w\in \cW$, Lemma~\ref{lem: sup expectation for Polish} shows that for every $y\in \cY$, there exists a universally measurable transport map $m_y:\cX\to\cX$  satisfying $d(x,m(x))\leq \epsilon$ for all $x\in\cX$ such that the following holds.
\begin{align*}
    R_{\Gamma_\epsilon}(\ell, w)
    &= 
    \sup_{\gamma\in \Gamma_\epsilon}\E_{(x',y)\sim \rho_y\rho^\gamma_{x'|y}}\left[  \ell((x', y),w) \right]\\
    &= \E_{(x,y))\sim \rho_y\rho_{x|y}} \left[ \ell((m_y(x),y),w) \right]\\
    &\leq \sup_{F\in F_\epsilon}\E_{(x,y)\sim \rho}\left[  \ell((f_y(x),y),w) \right]\\
    &= R_{F_\epsilon}(\ell, w).
\end{align*}
Combining the above inequality with \eqref{eq: adv risk gen loss inequality}, we have $R_{\oplus\epsilon}(\ell, w) =R_{F_\epsilon}(\ell, w) = R_{K_\epsilon}(\ell, w) =  R_{\Gamma_\epsilon}(\ell, w)$. 

\end{proof}

\subsection{$W_\infty$ Robustness in $\real^d$ via $2$-Alternating Capacities}\label{sec: capacities}

In this subsection, we establish a connection between adversarial risk and Choquet capacities \cite{Cho54} in $\real^d$. This connection allows us to extend Theorem~\ref{thm: Adv risk in Polish Space}  from Borel sets to the broader class of Lebesgue measurable sets. We will again use this connection for proving minimax theorems and existence of Nash equilibria in Section~\ref{sec: minimax in Rd}.
We begin with the following definitions. 

\begin{definition}[Capacity]
A set function $v:\cB(\cX)\to [0,1]$ is a \textit{capacity} if it satisfies the following conditions: (1) $v(\varnothing)=0$ and $v(\mathcal{X})=1$; (2) For $A,B\in \cB(\cX)$, $A\subseteq B\implies v(A)\leq v(B)$; (3) $A_n\uparrow A\implies v(A_n)\uparrow v(A)$; and (4) $F_n\downarrow F$, $F_n$ closed $\implies v(F_n)\downarrow v(F)$.
\end{definition}

\begin{definition}[$2$-Alternating Capacity]
A capacity $v$ defined on the measure space $(\cX, \cB(\cX))$ is called $2$-alternating if $v(A \cup B) + v(A\cap B) \leq v(A)+v(B)$
for all $A,B\in \cB(\cX)$.
\end{definition}

For any compact set of probability measures $\Xi\subseteq\cP(\cX)$, the upper probability defined as $v(A) = \sup_{\mu\in \Xi}\mu(A)$ is a capacity \cite{HubStr73}.  The upper probability of $\epsilon$-neighborhoods of a $\mu\in \cP(\cX)$ defined using either the total variation metric or the Levy-Prokhorov metric can be shown to be a $2$-alternating capacity \cite{HubStr73}. The following lemma shows that $A\mapsto \mu(A^{\oplus\epsilon})$ is a $2$-alternating capacity under some conditions.
\begin{lemma}\label{lem: set expansion is 2-alternating capacity}
Let $(\cX,d) =(\mathbb{R}^d, \|\cdot\|)$.
Let $\mu\in \overline{\cP}(\cX)$ and let $\epsilon\geq 0$. Define a set function $v$ on $\cX$ such that for any $A\in \cL(\cX)$, $v(A) := \mu(A^{\oplus\epsilon})$. Then $v$ is a $2$-alternating capacity.
\end{lemma}

The proof of Lemma~\ref{lem: set expansion is 2-alternating capacity} is included in Appendix~\ref{app: capacities}.

Now we relate the capacity defined in Lemma~\ref{lem: set expansion is 2-alternating capacity} to the $W_\infty$ metric. Since the $\epsilon$-neighborhood of a $\mu\in \cP(\cX)$ in $W_\infty$ metric is a compact set of probability measures \cite{YueEtal21}, the upper probability over this $W_\infty$ $\epsilon$-ball is a capacity. The following lemma shows that it is a  $2$-alternating capacity, and identifies it with the capacity defined in Lemma~\ref{lem: set expansion is 2-alternating capacity}.
\begin{lemma}\label{lem: sup measure for measurable sets}
Let $(\cX,d) =(\mathbb{R}^d, \|\cdot\|)$. Let $\mu\in \overline{\cP}(\cX)$. Then for any $A\in \cL(\cX)$, $\sup_{W_\infty(\mu, \mu')\leq\epsilon} \mu'(A) = \mu(A^{\oplus\epsilon})$.
Moreover, the supremum in the previous equation is attained.
\end{lemma}

The proof of Lemma~\ref{lem: sup measure for measurable sets} is included in Appendix~\ref{app: capacities}.
Lemma~\ref{lem: sup measure for measurable sets} plays a similar role to Lemma~\ref{lem: sup measure for Borel sets} in proving the following equivalence between adversarial robustness and $W_\infty$ robustness.
\begin{theorem}\label{thm: Adv risk in Rd}
Let $(\cX,d) =(\mathbb{R}^d, \|\cdot\|)$.
Let $p_0, p_1\in \overline{\cP}(\cX)$ and let $\epsilon\geq 0$.
Then for any $A\in \cL(\cX)$,
$R_{\oplus\epsilon}(\ell_{0/1}, A) = R_{\Gamma_\epsilon}(\ell_{0/1}, A)$,
and the supremum over $p_0'$ and $p_1'$ in $R_{\Gamma_\epsilon}(\ell_{0/1}, A)$ is attained.
\end{theorem}
\begin{proof}
Observe that
\begin{align*}
    R_{\Gamma_\epsilon}(\ell_{0/1}, A) 
    &= \frac{T}{T+1} \left( \sup_{W_\infty(p_0, p_0')\leq\epsilon}  p_0'(A)\right) + \frac{1}{T+1} \left( \sup_{W_\infty(p_1, p_1')\leq\epsilon}  p_1'((A^c))\right)\\
    &\stackrel{(*)}{=} \frac{T}{T+1} p_0(A^{\oplus\epsilon}) + \frac{1}{T+1} p_1((A^c)^{\oplus\epsilon})\\
    &= R_{\oplus\epsilon}(\ell_{0/1}, A),
\end{align*}
where $(*)$ follows from Lemma~\ref{lem: sup measure for measurable sets}. 
By Lemma~\ref{lem: sup measure for measurable sets} again, the supremum over $p_0'$ and $p_1'$ in $R_{\Gamma_\epsilon}(\ell_{0/1}, A)$ is attained.
\end{proof}

Unlike Theorem~\ref{thm: Adv risk in Polish Space}, Theorem~\ref{thm: Adv risk in Rd} does not show the equivalence of $R_{F_\epsilon}(\ell_{0/1}, A)$ with the other definitions under the relaxed assumption of $A\in \cL(\cX)$. This is because Lemma~\ref{lem: sup measure for measurable sets} does not provide a push-forward map $\phi$ such that $\mu^* = \phi_{\sharp \mu}$ with $\mu^*$ attaining the supremum over the $W_\infty$ ball.

\section{Optimal Adversarial Risk via Generalized Strassen's Theorem}\label{sec: strassen}

In Section~\ref{sec: W_infty}, we analyzed adversarial risk for a specific decision region $A\in \cB(\cX)$. In this section, we analyze infimum of adversarial risk over all possible decision regions; i.e., the optimal adversarial risk. We show that optimal adversarial risk in binary classification with unequal priors is characterized by an unbalanced optimal transport cost between data-generating distributions. Our main technical lemma generalizes Strassen's theorem to unbalanced optimal transport. We present this result in Subsection~\ref{sec: unbalanced OT} and present our characterization of optimal adversarial risk in Subsection~\ref{sec: opt adv risk unequal priors}.

\subsection{Unbalanced Optimal Transport and Generalized Strassen's Theorem}\label{sec: unbalanced OT}

Recall from Section~\ref{sec: defns} that the optimal transport cost $D_\epsilon$ characterizes the optimal adversarial risk in binary classification for equal priors. 
The following result gives an alternative characterization of $D_\epsilon$.
\begin{proposition}[Strassen's theorem][Corollary 1.28 in \cite{Vil03}]\label{prop: strassens theorem}
Let $\mu, \nu\in \cP(\cX)$. Let $\epsilon\geq 0$. Then
\begin{align}\label{eq: strassen}
    \sup_{A\in\cB(\cX)}
    \mu(A) - \nu(A^{2\epsilon})
    = D_\epsilon(\mu, \nu).
\end{align}
\end{proposition}
Proposition~\ref{prop: strassens theorem} is a special case of Kantorovich-Rubinstein duality \cite{Vil03} applied to $\{0,1\}$-valued cost functions.
We now generalize this result to measures with unequal masses. We begin with some definitions that generalize the concepts we introduced in Subsection~\ref{sec: OT preliminaries}. 

Let $\mu, \nu \in \cM(\cX)$ be such that $\mu(\cX)\leq \nu(\cX)$. 
A \textit{coupling} between $\mu$ and $\nu$ is a measure $\pi \in \cM(\cX^2)$ such that for any $A\in \cB(\cX)$,
$\pi(A\times\cX) = \mu(A)$ and 
$\pi(\cX\times A) \leq \nu(A)$.
The set $\Pi(\mu, \nu)$ is defined to be the set of all couplings between $\mu$ and $\nu$.
For a cost function $c:\cX^2\to [0, \infty)$, the optimal transport cost between $\mu$ and $\nu$ under $c$ is defined as
$\cT_c(\mu, \nu) = \inf_{\pi\in \Pi(\mu, \nu)} \int_{\cX^2} c(x,x')d\pi(x,x')$.


\begin{theorem}[Generalized Strassen's theorem]\label{thm: generalized strassen}
Let $\mu, \nu\in \cM(\cX)$ be  such that $0< M = \mu(\cX)\leq \nu(\cX)$.  
Let $\epsilon>0$. 
Define $c_\epsilon:\cX^2\to \{0,1\}$ as
$c_\epsilon(x,x') = \1\{(x,x')\in \cX^2 : d(x,x')>2\epsilon\}$.
Then
\begin{align}\label{eq: generalized strassen}
    \sup_{A\in \cB(\cX)}
    \mu(A) - \nu(A^{2\epsilon})
    = \cT_{c_\epsilon}(\mu, \nu)
    = M \inf_{\nu'\in\cP(\cX): \nu'\preceq \nu/M}
    D_\epsilon \left( \mu/M, \nu'  \right).    
\end{align}
Moreover, the infimum on the right hand side is attained. (Equivalently, there is a coupling $\pi\in \Pi(\mu, \nu)$ that attains the unbalanced optimal transport cost $\cT_{c_\epsilon}(\mu, \nu)$.)
\end{theorem}
The proof of Theorem~\ref{thm: generalized strassen} is contained in Appendix~\ref{app: unbalanced OT}.
The leverages strong duality in linear programming. We first establish \eqref{eq: generalized strassen} for discrete measures on a finite support. We then apply the discrete result on a sequence of measures supported on a countable dense subset of the Polish space $\cX$. Using the tightness of finite measures on $\cX$, we construct an optimal coupling that achieves the cost $\cT_{c_\epsilon}(\mu, \nu)$ in \eqref{eq: generalized strassen}. We then show that the constructed coupling satisfies \eqref{eq: generalized strassen}. This proof strategy is adapted from the works of \cite{Dud10} and \cite{Sch74}.

\subsection{ Optimal Adversarial Risk for Unequal Priors}\label{sec: opt adv risk unequal priors}
Generalized Strassen's theorem involves closed set expansions. The following lemma allows us to switch to Minkowski set expansions.
\begin{lemma}\label{lem: pydijog general}
Let $\mu, \nu \in \overline{\cM}(\cX)$ and let $\epsilon\geq 0$.  Then $\sup_{A\in \cB(\cX)} \mu(A) - \nu(A^{2\epsilon})
    = \sup_{A\in \cB(\cX)} \mu(A^{\ominus\epsilon}) - \nu(A^{\oplus\epsilon})$. Moreover, the supremum on the right hand side of the above equality can be replaced by a supremum over closed sets.
\end{lemma}
The proof of Lemma~\ref{lem: pydijog general} is contained in Appendix~\ref{app: opt adv risk unequal priors}.
Using Lemma~\ref{lem: pydijog general} and the generalized Strassen's theorem, we show the following result on optimal adversarial risk for unequal priors, generalizing the result of \cite{PydJog20, BhaEtal19}.
\begin{theorem}\label{thm: adv_risk as D_ep}
Let $p_0, p_1\in \overline{\cP}(\cX)$ and let $\epsilon\geq 0$. Then,
\begin{align}\label{eq: adv_risk as D_ep}
    \inf_{A\in \cB(\cX)} R_{\oplus\epsilon}(\ell_{0/1}, A)
    = \frac{1}{T+1}\left[ 1- \inf_{q\in \cP(\cX): q\preceq Tp_0} D_\epsilon(q, p_1) \right].
\end{align}
Moreover, the infimum on the left hand side can be replaced by an infimum over closed sets.
\end{theorem}
\begin{proof}
\begin{align*}
    \inf_{A\in \cB(\cX)} R_{\oplus\epsilon}(\ell_{0/1}, A)
    &=  \inf_{A\in \cB(\cX)} \frac{1}{T+1} \left[ Tp_0(A^{\oplus\epsilon}) + p_1((A^c)^{\oplus\epsilon}) \right]\\
    &= \frac{1}{T+1} \left[ 1 -  \sup_{A\in \cB(\cX)} \left( p_1(A^{\ominus\epsilon}) - Tp_0(A^{\oplus\epsilon}) \right)  \right]\\
    &\stackrel{(i)}{=} \frac{1}{T+1} \left[ 1 -  \sup_{A\in \cB(\cX)} \left( p_1(A) - Tp_0(A^{2\epsilon}) \right)  \right]\\    
    &\stackrel{(ii)}{=} \frac{1}{T+1}\left[ 1- \inf_{\substack{q\in \cP(\cX):\\ q\preceq Tp_0}} D_\epsilon(q, p_1) \right],
\end{align*}
where $(i)$ follows from Lemma~\ref{lem: pydijog general} and $(ii)$ follows from Theorem~\ref{thm: generalized strassen}.
\end{proof}

Theorem~\ref{thm: adv_risk as D_ep} extends the result of \cite{PydJog21} in two ways: (1) the infimum is taken over all sets for which $R_{\oplus\epsilon}(\ell_{0/1}, A)$ is well-defined, instead of restricting to closed sets, and (2) the priors on both labels can be unequal. We also note that for  $(\cX,d) =(\mathbb{R}^d, \|\cdot\|)$, \eqref{eq: adv_risk as D_ep} holds with the infimum on the left hand side taken over all $A\in \cL(\cX)$.

\section{Minimax Theorems and Nash Equilibria}\label{sec: minimax}

In this section, we revisit the zero-sum game between the adversary and the algorithm introduced in Section~\ref{sec: defns}. Recall that for $A\in \cB(\cX)$ and $p_0', p_1'\in \overline{\cP}(\cX)$, the payoff function is given by
\begin{align}\label{eq: payoff function}
    r(A, p_0', p_1') = \frac{T}{T+1} p_0'(A) + \frac{1}{T+1} p_1'((A^c)).
\end{align}
The max-min inequality gives us
\begin{align}\label{eq: max-min inequality}
    \sup_{W_\infty(p_0, p_0'), W_\infty(p_1, p_1')\leq \epsilon}
    \inf_{A\in \cA}  r(A, p_0', p_1')
    \leq 
    \inf_{A\in \cB(\cX)} 
    \sup_{W_\infty(p_0, p_0'), W_\infty(p_1, p_1')\leq \epsilon}  r(A, p_0', p_1').
\end{align}

If the inequality in \eqref{eq: max-min inequality} is an equality, we say that the game has zero duality gap, and admits a value equal to either expression in \eqref{eq: max-min inequality}. In the equality setting, there is no advantage to a player making the first move. Our minimax theorems establish such an equality. If, in addition to having an equality in \eqref{eq: max-min inequality}, there exist $p_0^*, p_1^*\in \cP(\cX)$ that achieve the supremum on the left-hand side and  $A^*\in \cB(\cX)$ that achieves the infimum on the right-hand side, we say that $((p_0^*, p_1^*), A^*)$ is a pure Nash equilibrium of the game. On the other hand, we say that $((p_0^*, p_1^*), A^*)$ is a $\delta$-approximate pure Nash equilibrium of the game if the following inequality holds.
\begin{align*}
	\sup_{W_\infty(p_0, p_0'), W_\infty(p_1, p_1')\leq \epsilon}  r(A^*, p_0', p_1') - \delta \leq 	r(A^*, p_0^*, p_1^*) \leq \inf_{A\in \cA}  r(A, p_0^*, p_1^*) + \delta.
\end{align*}

In Section~\ref{sec: minimax in Rd}, we prove the minimax theorem and the existence of a pure Nash equilibrium in $\real^d$ using the theory of $2$-alternating capacities \cite{HubStr73} and the relation to adversarial risk from Section~\ref{sec: capacities}. Section~\ref{sec: minimax in Polish space} extends these results to more general Polish spaces with a ``midpoint property."

\subsection{Minimax Theorem in $\real^d$ via $2$-Alternating Capacities}\label{sec: minimax in Rd}

The following theorem proves the minimax equality and the existence of a Nash equilibrium for the adversarial robustness game in $\real^d$.

\begin{theorem}[Minimax theorem in $\real^d$]\label{thm: minimax theorem in R^d}
Let $(\cX,d) = (\real^d, \|\cdot\|)$. Let $p_0, p_1\in \overline{\cP}(\cX)$ and let $\epsilon\geq 0$. Define $r$ as in \eqref{eq: payoff function}.
Then,
\begin{align}\label{eq: min-max theorem in Rd}
    \sup_{W_\infty(p_0, p_0'), W_\infty(p_1, p_1')\leq \epsilon} \inf_{A\in \cL(\cX)}  r(A, p_0', p_1')
    = 
    \inf_{A\in \cL(\cX)} \sup_{W_\infty(p_0, p_0'), W_\infty(p_1, p_1')\leq \epsilon}  r(A, p_0', p_1').
\end{align}
Moreover, there exist $p_0^*, p_1^*\in \overline{\cP}(\cX)$  and  $A^*\in \cL(\cX)$ that achieve the supremum and infimum on the left and right hand sides of the above equation.
\end{theorem}
The proof of Theorem~\ref{thm: minimax theorem in R^d} is in Appendix~\ref{app: minimax in Rd}.
Crucial to the proof of Theorem~\ref{thm: minimax theorem in R^d} is Lemma~\ref{lem: set expansion is 2-alternating capacity}, which shows that the set-valued maps $A \mapsto p_0(A^{\oplus\epsilon})$ and $A^c \mapsto p_1((A^c)^{\oplus\epsilon})$ are $2$-alternating capacities.
The same proof technique is not applicable in general Polish spaces because the map $A\mapsto \mu(A^{\oplus\epsilon})$ is not a capacity for a general $\mu\in \overline{\cP}(\cX)$. This is because $A^{\oplus\epsilon}$ is not measurable for all $A\in \overline{\cB}(\cX)$.

\subsection{Minimax Theorem in Polish Spaces via Optimal Transport}\label{sec: minimax in Polish space}

We now extend the minimax theorem from $\real^d$ to general Polish spaces with the following property.

\begin{definition}[Midpoint property]
A metric space $(\cX, d)$ is said to have the midpoint property if for every $x_1, x_2\in \cX$, there exists $x\in \cX$ such that,
$d(x_1, x) = d(x, x_2) = d(x_1, x_2)/2$.
\end{definition}

Any normed vector space with distance defined as $d(x,x') = \|x-x'\|$ satisfies the midpoint property. An example of a metric space without this property is the discrete metric space where $d(x,x') = \1\{x\neq x'\}$. The midpoint property plays a crucial role in proving the following theorem, which shows that the $D_\epsilon$ transport cost between two distributions is the shortest total variation distance between their $\epsilon$-neighborhoods in $W_\infty$ metric. A similar result was also presented in \cite{Doh20}.

\begin{theorem}[$D_\epsilon$ as shortest $D_{TV}$ between $W_\infty$ balls]\label{thm: D_ep as shortest D_TV}
Let $(\cX,d)$ have the midpoint property. Let $\mu, \nu\in \overline{\cP}(\cX)$ and let $\epsilon\geq 0$.
Then $D_\epsilon(\mu, \nu) 
    = \inf_{W_\infty(\mu, \mu'), W_\infty(\nu, \nu')\leq \epsilon}
    D_{TV}(\mu', \nu')$. Moreover, the infimum over $D_{TV}$ in the above equation is attained.
\end{theorem}

The proof of Theorem~\ref{thm: D_ep as shortest D_TV} is in Appendix~\ref{app: minimax in Polish space}.
The following theorem uses Theorem~\ref{thm: D_ep as shortest D_TV} to prove the minimax equality and the existence of a Nash equilibrium for any Polish space with the midpoint property for the case of equal priors.

\begin{theorem}[Minimax theorem for equal priors]\label{thm: minimax equal priors}
Let $(\cX,d)$ have the midpoint property. Let $p_0, p_1\in \overline{\cP}(\cX)$ and let $\epsilon\geq 0$.
Define $r$ as in \eqref{eq: payoff function} with $T=1$.
Then
\begin{align}
    \sup_{W_\infty(p_0, p_0'), W_\infty(p_1, p_1')\leq \epsilon}
    \inf_{A\in \cB(\cX)}  r(A, p_0', p_1')
    = 
    \inf_{A\in \cB(\cX)} 
    \sup_{W_\infty(p_0, p_0'), W_\infty(p_1, p_1')\leq \epsilon}  r(A, p_0', p_1').
\end{align}
Moreover, there exist $p_0^*, p_1^*\in \cP(\cX)$ that achieve the supremum on the left hand side of the above equation.
\end{theorem}
\begin{proof}
We have the following series of equalities.
\begin{align*}
    \inf_{A\in \cB(\cX)} \sup_{\substack{W_\infty(p_0, p_0')\leq \epsilon\\W_\infty(p_1, p_1')\leq \epsilon}}  r(A, p_0', p_1')
    &= \inf_{A\in \cB(\cX)} R_{\Gamma_\epsilon}(\ell_{0/1}, A)\\
    &\stackrel{(i)}{=} \inf_{A\in \cB(\cX)} R_{\oplus\epsilon}(\ell_{0/1}, A)\\
    &\stackrel{(ii)}{=} \frac{1}{2} \left[1-  D_\epsilon(p_0, p_1) \right],
\end{align*}
and 
\begin{align*}
    \sup_{\substack{W_\infty(p_0, p_0')\leq \epsilon\\W_\infty(p_1, p_1')\leq \epsilon}}
    \inf_{A\in \cB(\cX)}  r(A, p_0', p_1')    
    &\stackrel{(iii)}{=} \sup_{\substack{W_\infty(p_0, p_0')\leq \epsilon\\W_\infty(p_1, p_1')\leq \epsilon}}
    \frac{1}{2} \left[1-  D_{TV}(p_0', p_1') \right]\\
    &= \frac{1}{2} 
    \left[1- 
    \inf_{\substack{W_\infty(p_0, p_0')\leq \epsilon\\W_\infty(p_1, p_1')\leq \epsilon}}  D_{TV}(p_0', p_1') \right],    
\end{align*}
where (i) follows from Theorem~\ref{thm: Adv risk in Polish Space}, (ii)  from Theorem~\ref{thm: adv_risk as D_ep}, and $(iii)$ again from Theorem~\ref{thm: adv_risk as D_ep} with $\epsilon=0$. The expressions on the right extremes of the above equations are equal by Theorem~\ref{thm: D_ep as shortest D_TV}. The existence of  $p_0^*, p_1^*\in \overline{\cP}(\cX)$ follows Theorem~\ref{thm: D_ep as shortest D_TV}.
\end{proof}

To prove the minimax theorem for unequal priors, we need the following generalization of Theorem~\ref{thm: D_ep as shortest D_TV} to finite measures of unequal mass.
\begin{lemma}\label{lem: layered balls}
Let $p_0, p_1\in \overline{\cP}(\cX)$ and let $\epsilon\geq 0$. Then for $T\geq 1$,
\begin{align}\label{eq: layered balls}
	\inf_{q\in \overline{\cP}(\cX): q\preceq Tp_0} D_\epsilon(q, p_1)
	& = 
    \inf_{q\in \overline{\cP}(\cX): q\preceq Tp_0}
    \inf_{W_\infty(q, q'), W_\infty(p_1, p_1')\leq \epsilon} 
    D_{TV}(q', p_1') \notag \\
    & = 
    \inf_{W_\infty(p_0, p_0'), W_\infty(p_1, p_1')\leq \epsilon} 
    \inf_{q'\in \overline{\cP}(\cX): q'\preceq Tp_0'}
    D_{TV}(q', p_1')
\end{align}
\end{lemma}
The proof of Lemma~\ref{lem: layered balls} is contained in Appendix~\ref{app: minimax in Polish space}.

\begin{figure}
    \centering
    \includegraphics[scale=0.7]{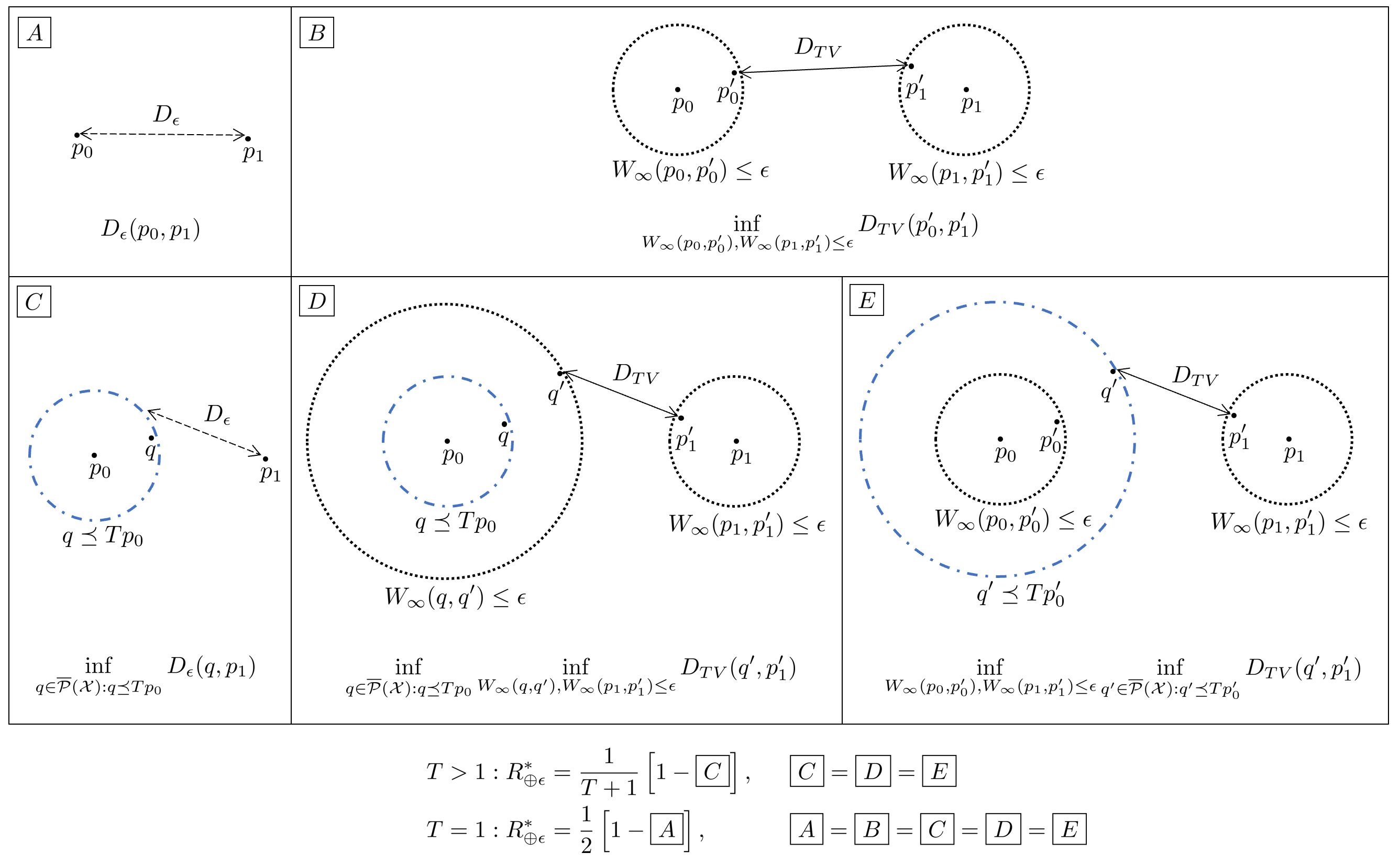}
    \caption{Illustration of various equivalent formulations of the optimal adversarial risk. The equalities summarize the results of Section~\ref{sec: strassen} and Section~\ref{sec: minimax}. For equal priors ($T=1$), $\boxed{A}$ and $\boxed{B}$ denote two ways of obtaining the optimal adversarial risk, $R^*_{\oplus\epsilon}$: 1) $\boxed{A}$, which denotes the $D_\epsilon$ cost between the true label distributions $p_0$ and $p_1$, and 2) $\boxed{B}$, which denotes the shortest total variation distance between $\infty$-Wasserstein balls of radius $\epsilon$ around $p_0$ and $p_1$. For unequal priors ($T>1$), $\boxed{C}, \boxed{D}$ and $\boxed{E}$ denote three equivalent ways of obtaining $R^*_{\oplus\epsilon}$. The black dotted balls denote $\infty$-Wasserstein balls and the blue dashed balls denote sets defined using stochastic domination. The order in which the two types of balls appear around $p_0$ is reversed between $\boxed{D}$ and $\boxed{E}$.}
    \label{fig: layered balls}
\end{figure}

Now, we prove the minimax equality for unequal priors.
\begin{theorem}[Minimax theorem for  unequal priors]\label{thm: minimax unequal priors}
Let $(\cX,d)$ have the midpoint property. Let $p_0, p_1\in \overline{\cP}(\cX)$ and let $\epsilon\geq 0$.
For $T>0$, define $r$ as in \eqref{eq: payoff function}.
Then
\begin{align}
    \sup_{W_\infty(p_0, p_0'), W_\infty(p_1, p_1')\leq \epsilon}
    \inf_{A\in \cB(\cX)}  r(A, p_0', p_1')
    = 
    \inf_{A\in \cB(\cX)} 
    \sup_{W_\infty(p_0, p_0'), W_\infty(p_1, p_1')\leq \epsilon}  r(A, p_0', p_1').
\end{align}
\end{theorem}
\begin{proof}
Without loss of generality, we assume $T\geq 1$. (If $T<1$, we simply repeat the proof with labels $0$ and $1$ swapped.) We have

\begin{align*}
    \inf_{A\in \cB(\cX)} \sup_{\substack{W_\infty(p_0, p_0')\leq \epsilon\\W_\infty(p_1, p_1')\leq \epsilon}}  r(A, p_0', p_1')
    &= \inf_{A\in \cB(\cX)} R_{\Gamma_\epsilon}(\ell_{0/1}, A)\\
    &\stackrel{(i)}{=} \inf_{A\in \cB(\cX)} R_{\oplus\epsilon}(\ell_{0/1}, A)\\
    &\stackrel{(ii)}{=} \frac{1}{T+1} \left[1- \inf_{\substack{q\in \overline{\cP}(\cX):\\ q\preceq Tp_0}} D_\epsilon(p_0, p_1) \right]\\
    &\stackrel{(iii)}{=} \frac{1}{T+1} \left[1-     
    \inf_{\substack{W_\infty(p_0, p_0')\leq \epsilon\\W_\infty(p_1, p_1')\leq \epsilon}} 
    \inf_{\substack{q'\in \overline{\cP}(\cX):\\ q'\preceq Tp_0'}}
    D_{TV}(q', p_1') \right]\\    
    &= \sup_{\substack{W_\infty(p_0, p_0')\leq \epsilon\\W_\infty(p_1, p_1')\leq \epsilon}}
    \frac{1}{T+1} 
    \left[1- \inf_{\substack{q'\in \overline{\cP}(\cX):\\ q'\preceq Tp_0'}} D_{TV}(q', p_1') \right]\\
    &\stackrel{(iv)}{=} \sup_{\substack{W_\infty(p_0, p_0')\leq \epsilon\\W_\infty(p_1, p_1')\leq \epsilon}}
    \inf_{A\in \cB(\cX)}  r(A, p_0', p_1'),    
\end{align*}
where (i) follows from Theorem~\ref{thm: Adv risk in Polish Space}, (ii)  from Theorem~\ref{thm: adv_risk as D_ep}, (iii)  from Lemma~\ref{lem: layered balls} and (iv) follows again from Theorem~\ref{thm: adv_risk as D_ep} with $\epsilon=0$.
\end{proof}

\begin{remark*}
Unlike Theorem~\ref{thm: minimax theorem in R^d}, Theorems~\ref{thm: minimax equal priors} and \ref{thm: minimax unequal priors} do not guarantee the existence of an optimal decision region $A^*$. 
While Theorem~\ref{thm: minimax equal priors} guarantees the existence of worst-case pair of perturbed distributions $p_0^*, p_1^*$, Theorem~\ref{thm: minimax unequal priors} does not do so. Nevertheless, a $\delta$-approximate pure Nash equilibrium exists in all the cases. This is in sharp contrast with the non-existence of Nash equilibrium proven in  \cite{PinEtal20}. The result of \cite{PinEtal20} is valid for a ``regularized'' adversary, where the point-wise budget constraint $d(x,x')\leq\epsilon$ is replaced with a regularization term added to the adversarial risk formulation. Our Nash equilibrium result holds for the standard formulation of adversarial risk as in \cite{MadEtal18, ShaEtal18}, without the need for a  regularization term.
\end{remark*}

\begin{remark*}
A recent work \cite{MeuEtal21} shows the existence of mixed Nash equilibrium  for randomized classifiers parametrized by points in a Polish space. Other works \cite{PinEtal20, BosEtal20} consider a similar setup, but with a ``regularized'' adversary. The equilibrium analysis in these works uses Fan's minimax theorem with concave-convex condition. Since we consider non-parametric classifiers represented by arbitrary decision regions, Fan's theorem is inapplicable in our setting. Instead, we use tools from Huber's 2-alternating capacities for $\real^d$, and the generalized Strassen's duality theorem for general Polish spaces. 
The connection with Huber's capacities (which we prove in Lemma~\ref{lem: set expansion is 2-alternating capacity}) and the generalization of Strassen's theorem (Theorem~\ref{thm: generalized strassen}) are both novel to the best of our knowledge.
\end{remark*}

\section{Discussion}\label{sec: discussion}
We examined different notions of adversarial risk and laid down the conditions under which these definitions are equivalent. 
By verifying the conditions in Sections~\ref{sec: well-defined} and $\ref{sec: W_infty}$, researchers may use different definitions interchangeably. 

We analyzed optimal adversarial risk for (non-parametric) decision region-based classifiers. Using a formulation of optimal transport between finite measures of unequal mass, we extended the optimal transport based characterization of adversarial risk of \cite{PydJog20, BhaEtal19} to unequal priors by generalizing Strassen's theorem. This may find  applications in the study of excess cost optimal transport \cite{YuTan18, Yu19}. A recent work \cite{TriMur20} obtains a different characterization of optimal adversarial risk using optimal transport on the product space $\cX\times\cY$ where $\cY$ is the label space. Further, they show the evolution of the optimal classifier $A^*$ as $\epsilon$ grows, in terms of a mean curvature flow. This raises an interesting  question on the evolution of the optimal adversarial distributions $p_0^*, p_1^*\in \overline{\cP}(\cX)$ with $\epsilon$. 

We proved a minimax theorem for adversarial robustness game and the existence of a Nash equilibrium. We constructed the worst-case pair of distributions $p_0^*, p_1^*\in \overline{\cP}(\cX)$ in terms of true data distributions and showed that their total variation distance  gives the optimal adversarial risk. Identifying worst case distributions could lead to a new approach to developing robust algorithms.

We used Choquet capacities for results in $\real^d$ and measurable selections in Polish spaces. Specifically, we showed that the measure of $\epsilon$-Minkowski expansion is a $2$-alternating capacity. This connection could help generalize our results to total variation and Prokhorov distance based contaminations. 

We largely focused on the binary classification setup with $0$-$1$ loss function. While we extended our results on measurability and relation to $\infty$-Wasserstein distributional robustness to more general loss functions and a multi-class setup, it is unclear how our results on generalized Strassen's theorem and Nash equilibria can be extended further. Our results on various equivalent formulations of optimal adversarial risk are specific to adversarial perturbations (or equivalently, $\infty$-Wasserstein distributional perturbations). 
An interesting open question is whether these results hold for more general perturbation models. 

\section*{Acknowledgements}

The authors acknowledge support from NSF grants CCF-1841190 and CCF-1907786, and from the University of Cambridge. The authors also thank anonymous reviewers for their insightful comments on a version of this paper that was presented at NeurIPS 2021~\cite{PydiJog21}.

\bibliographystyle{plain}
\bibliography{ref}


\appendix

\section{Preliminary Lemmas}\label{app: preliminary lemmas}

\begin{lemma}\label{lem: Minkowski sum of union}
Let $A_n\in \cB(\cX)$ for $n\in \{1, 2, \ldots\}$. Then, 
\begin{align*}
    (\cup_n A_n)^{\oplus\epsilon} = \cup_n A_n^{\oplus\epsilon},\\
    (\cap_n A_n)^{\oplus\epsilon} \subseteq \cap_n A_n^{\oplus\epsilon}.    
\end{align*}
\end{lemma}
\begin{proof}
Suppose $a\in (\cup_n A_n)^{\oplus\epsilon}$. Then there exists $a_i\in A_i$ for some $i\in \{1, 2, \ldots\}$ such that $d(a,a_i)\leq \epsilon$. Hence, $a\in A_i^{\oplus\epsilon}\subseteq \cup_n A_n^{\oplus\epsilon}$. Therefore, $(\cup_n A_n)^{\oplus\epsilon} \subseteq \cup_n A_n^{\oplus\epsilon}$.

Suppose $b\in \cup_n A_n^{\oplus\epsilon}$. Then $b\in A_j^{\oplus\epsilon}$ for some $j\in \bN$. So there must exist $b'\in A_j$ such that $d(b,b')\leq \epsilon$. Since $b'\in \cup_n A_n$, we get that $b\in (\cup_n A_n)^{\oplus\epsilon}$. Therefore, $\cup_n A_n^{\oplus\epsilon}\subseteq (\cup_n A_n)^{\oplus\epsilon}$.

Suppose $c\in (\cap_n A_n)^{\oplus\epsilon}$. Then there exists $c'\in \cap_n A_n$  such that $d(c, c')\leq \epsilon$.
Since $c'\in A_n$ for all $n\in \{1, 2, \ldots\}$, $c\in A_n^{\oplus\epsilon}$ for all  $n\in \{1, 2, \ldots\}$. Hence, $c\in \cap_n A_n^{\oplus\epsilon}$. Therefore, $(\cap_n A_n)^{\oplus\epsilon} \subseteq \cap_n A_n^{\oplus\epsilon}$.
\end{proof}

\begin{lemma}\label{lem: convergence of closed sets}
Let $(F_n)$ be a sequence of closed sets  in $\cX$ such that $F_k\supseteq F_{k+1}$ for $k\in \N$. Then,
\begin{align*}
    (\cap_n F_n)^{\oplus\epsilon} = \cap_n F_n^{\oplus\epsilon}.
\end{align*}
\end{lemma}
\begin{proof}
Suppose $x\in (\cap_n F_n)^{\oplus\epsilon}$. Then there exists $x'\in \cap_n F_n$ such that $d(x,x')\leq\epsilon$. 
Since $x'\in F_n$ for all $n\in \bN$, $x\in F_n^{\oplus\epsilon}$ for all $n\in \bN$. 
Hence, $x\in \cap_n F_n^{\oplus\epsilon}$ and therefore $(\cap_n F_n)^{\oplus\epsilon} \subseteq \cap_n F_n^{\oplus\epsilon}$. 
We will now show the set inclusion in the opposite direction.

Let $x\in \cap_n F_n^{\oplus\epsilon}$. Then $x\in F_n^{\oplus\epsilon}$ for all $n\in \bN$. Hence, there exists $x_n\in F_n$ such that $d(x,x_n)\leq \epsilon$ for all $n\in \bN$.
Since $(x_n)$ is a bounded sequence, it has a subsequence $(x_{n_k})$ that converges to some $x^*$.
We claim that $x^*\in F \defn \cap_n F_n$. Indeed, for any $m\in \bN$, the tail of the subsequence $(x_{n_k})$ with indices greater than $m$ is contained in $F_m$. Since $F_m$ is closed, $x^*$ must be in $F_m$. Since the choice of $m$ was arbitrary, $x^*\in \cap_m F_m = F$.
Hence, $x\in F^{\oplus\epsilon}$ because $d(x,x^*)\leq \epsilon$. Therefore, 
$\cap_n F_n^{\oplus\epsilon} \subseteq F^{\oplus\epsilon}$.
\end{proof}

\begin{lemma}\label{lem: closed set}
Let $A\in\cB(\cX)$. Let $(\gamma_n)_{n=1}^\infty$ be a non-negative, monotonically decreasing sequence converging to $0$. Let $\overline{A}$ denote the closure of $A$ in $\cX$. Then, $A^{\gamma_n} \downarrow \overline{A}$.
\end{lemma}
\begin{proof}
We know $\overline A \subseteq \overline A^{\gamma_n} = A^{\gamma_n}$ for all $n$. Hence $\overline A\subseteq \lim_{n\to \infty} \bigcap_{k=1}^n A^{\gamma_k}$.

Suppose $x \in \lim_{n\to \infty} \bigcap_{k=1}^n A^{\gamma_k}$. Then it must be that $d(x, \overline A) = 0$ because otherwise $x$ would not lie in $A^{\gamma_n}$ for all large enough $n$. Since $d(x, \overline A) = 0$, we can find a sequence of points in $\overline A$ that tend to $x$. But since $\overline A$ is closed, we must have $x \in \overline A$. Hence, $\lim_{n\to \infty} \bigcap_{k=1}^n A^{\gamma_k} \subseteq \overline A$.
\end{proof}

\begin{lemma}\label{lem: set expansion property}
Let $\epsilon_1>\epsilon_2>0$ and $A\in \cB(\cX)$. Then for any $\delta\in (0, \epsilon_1-\epsilon_2)$, $A^{\epsilon_1-\epsilon_2-\delta}\subseteq (A^{\epsilon_1})^{-\epsilon_2}$.
\end{lemma}
\begin{proof}
Recall that for $\epsilon>0$, $A^{-\epsilon} = ((A^c)^\epsilon)^c$. From the definition, $x\in A^{-\epsilon}$ if and only if $d(x, A^c)>\epsilon$.

Let $\delta\in (0, \epsilon_1-\epsilon_2)$ and $x\in A^{\epsilon_1-\epsilon_2-\delta}$. Then, $d(x,A)\leq \epsilon_1-\epsilon_2-\delta$. Consider any $y\in (A^{\epsilon_1})^c$. Then, $d(y, A)>\epsilon_1$. By the triangle inequality,
\begin{align*}
d(x,y) \geq d(y, A) - d(x, A) > \epsilon_1 - (\epsilon_1-\epsilon_2-\delta) = \epsilon_2+\delta.
\end{align*}
Hence, 
\begin{align*}
d(x, (A^{\epsilon_1})^c) = \inf_{y\in (A^{\epsilon_1})^c} d(x,y) \geq \epsilon_2+\delta >\epsilon_2.
\end{align*}
Therefore, $x\in (A^{\epsilon_1})^{-\epsilon_2}$. 
\end{proof}

\begin{lemma}\label{lem: oplus expansion as sup}
Let $A\in \cB(\cX)$. Then, $\1\{x\in A^{\oplus\epsilon}\} = \sup_{x'\in B_\epsilon(x)} \1\{x'\in A\}$.
\end{lemma}
\begin{proof}
Suppose $x\in A^{\oplus\epsilon}$. 
Then there exists $x'\in A$ such that $x'\in B_\epsilon(x)$. Hence, $\sup_{x'\in B_\epsilon(x)} \1\{x'\in A\} = 1$.

Suppose $x\in \cX$ is such that $\sup_{x'\in B_\epsilon(x)} \1\{x'\in A\} = 1$. Then there is a sequence $(x_n)_{n=1}^\infty$ such that $d(x,x_n)\leq\epsilon$ and $x_n\in A$ for all $n$. Since $(x_n)$ is a bounded sequence in a closed set, $B_\epsilon(x)$, it has a subsequence that converges to some $x^*$ such that $d(x,x^*)\leq \epsilon$ and $x^*\in A$. Hence, $x\in B_\epsilon(x^*) \subseteq A^{\oplus\epsilon}$.
\end{proof}

\begin{lemma}\label{lem: level set expansion}
For any real-valued function $f:\cX\to \real$ and any $t\in \real$, 
\begin{align*}
    \left\{x\in \cX: \sup_{d(x,x')\leq \epsilon} f(x') > t\right\}
    = \{x\in \cX: f(x) >t \}^{\oplus\epsilon}.
\end{align*}
\end{lemma}
\begin{proof}
Suppose $a\in \{x\in \cX: f(x) >t \}^{\oplus\epsilon}$. Then there exists $a'\in \cX$ such that $f(a') >t$ and $d(a,a')\leq \epsilon$. Hence, 
$\sup_{d(a,x')\leq \epsilon} f(x') \geq f(a') >t$.
Therefore, $a\in \left\{x\in \cX: \sup_{d(x,x')\leq \epsilon} f(x') > t\right\}$.

Suppose $b\in \left\{x\in \cX: \sup_{d(x,x')\leq \epsilon} f(x') > t\right\}$. Then there exists $b'\in \cX$ such that $f(b') >t$ and $d(b,b')\leq \epsilon$. Hence, $b\in \{x\in \cX:f(x) >t \}^{\oplus\epsilon}$.
\end{proof}

\section{Proofs from Section~\ref{sec: well-defined}}\label{app: well-defined}

\subsection{Proofs from Section~\ref{sec: well-defined 0-1 loss}}\label{app: well-defined 0-1 loss}

\begin{proof}[Proof of Lemma~\ref{lem: non-measurable}]
We prove the above statement by using a counterexample motivated from Example 2.4 in \cite{LuiEtal14}.
For any $\epsilon>0$, there exists a Borel measurable set $S\subseteq [-\epsilon, \epsilon]^2$ such that its projection onto the first coordinate is not Borel measurable (\cite{LuiEtal14}, Theorem 6.7.2 and Theorem 6.7.11 in \cite{Bog07}). That is,
$S\in \cB(\real^2)$ but $S_1 \defn \{x_1\in \real: (x_1, x_2)\in S\} \notin \cB(\real)$. 

Define a homeomorphism $\phi: \real^3\to\real^3$ as $\phi(x_1, x_2, x_3) \defn (x_1, x_2, \sqrt{\epsilon^2-x_2^2})$. $\phi$ maps the plane $[-\epsilon, \epsilon]^2\times \{0\}$ onto the half-cylinder, $\{(x_1, x_2, x_3)\in \real^3: x_1\in [-\epsilon, \epsilon], x_2^2+x_3^2 = \epsilon^2, x_3\geq 0\}$, of radius $\epsilon$.
Let $A \defn \phi(S \times \{0\})$.
Then $A \in \cB(\real^3)$ because $S\times\{0\}\in \cB(\real^3)$. We have the following equality.
\begin{align*}
    A^{\oplus\epsilon}\cap (\real\times\{0\}^2) = S_1 \times \{0\}^2
\end{align*}
Suppose $A^{\oplus\epsilon}\in \cB(\real^3)$. Then the above equality implies that $S_1\in \cB(\real)$ contradicting our choice of $S$. Hence, $A^{\oplus\epsilon}\notin \cB(\real^3)$.

\end{proof}

\begin{proof}[Proof of Lemma~\ref{lem: Borel set expansion}]
Recall that an analytic set is a continuous image of a Borel set in a Polish space. Although an analytic set need not be Borel measurable, it is always universally measurable, i.e., measurable with respect to any measure defined on a complete measure space \cite{BerShr96}.

We will now show that if $A\in \cB(\cX)$, then $A^{\oplus\epsilon}$ is an analytic set, thus showing that it is measurable in the complete measure space $(\cX, \overline{\cB}(\cX))$.

Define $D = \{(x,x')\in \cX^2: d(x,x')\leq\epsilon\}$. $D$ is Borel measurable because it is the preimage of the Borel set $(-\infty, \epsilon]$ under the Borel measurable function $d$. Define $f: D\to \real$ as $f(x,x') = -\1\{x'\in A\}$. For $c\in \real$, we have the following.
\begin{align*}
    \{(x,x')\in \cX^2: f(x,x')<c\}
    =
    \begin{cases}
    \phi & c\leq -1,\\
    (\cX\times A)\cap D & c\in (-1, 0],\\
    \cX^2 & c>0.
    \end{cases}
\end{align*}
Since $A\in \cB(\cX)$ and $D\in \cB(\cX^2)$, $(\cX\times A)\cap D\in \cB(\cX^2)$. Hence, by Definition 7.21 in \cite{BerShr96}, $f$ is a  lower semianalytic function. 
By Proposition 7.47 in \cite{BerShr96}, the function $f^*: \cX\to \real$ defined as $f^*(x) \defn \inf_{x'\in B_\epsilon(x)} f(x,x')$ is lower semianalytic. By Lemma~\ref{lem: oplus expansion as sup}, we have
\begin{align*}
    f^*(x) 
    = \inf_{x'\in B_\epsilon(x)} -\1\{x'\in A\}
    = - \sup_{x'\in B_\epsilon(x)} \1\{x'\in A\}
    = - \1\{x\in A^{\oplus\epsilon}\}.
\end{align*}
By Definition 7.21 in \cite{BerShr96}, it follows that $A^{\oplus\epsilon}$ is an analytic set. By Corollary 7.42.1 in \cite{BerShr96}, $A^{\oplus\epsilon}\in \overline{\cB}(\cX)$.

\end{proof}

\begin{proof}[Proof of Lemma~\ref{lem: A_ep_porous}]
Let $\beta = 1/4$. Take any $e\in E$. Since $E = A^{\epsilon}\backslash A^{\epsilon)}$, we have the following two implications: 1) $E\subseteq A^{\epsilon}$ which implies that $d(e, A)\leq\epsilon$, and 2) $E\cap A^{\epsilon)} = \varnothing$ which implies that $d(e, A)>\epsilon$. Combining the two implications, we get that $d(e, A) = \epsilon$. Hence, for every $r\in (0, \epsilon]$, there must exist an $a_r\in A$ such that $\epsilon\leq \|e-a_r\| < \epsilon + r/4$.
We pick an $x'\in \cX$ on the line segment joining $a_r$ and $x$ as follows.
\begin{align*}
    t &\defn \frac{r}{2\|e-a_r\|}, \\
    x' &\defn ta_r + (1-t)e.
\end{align*}
Since $ \|e-a_r\| \in [\epsilon, \epsilon+r/4)$ and $r\in (0, \epsilon]$, it is clear that $t\in (0, 1/2)$.
From the definition of $x'$, it follows that $\|x' - e\| = t\|e-a_r\|  = r/2$. 
We will now show that $B_{\beta r}(x')\subseteq B_r(e)\backslash E$.
For any $y\in B_{\beta r}(x')$, we have the following.
\begin{align*}
    \|y - e\| \leq \|y-x'\| + \|x'-e\| \leq \beta r + r/2 < r.
\end{align*}
Hence, $y\in B_r(e)$. Moreover,
\begin{align*}
    \|y - a_r\| \leq \|y-x'\| + \|x'-a_r\| \leq \beta r + (\|e-a_r\| - r/2) < \epsilon.
\end{align*}
Hence, $y\in A^{\epsilon)}$ and so $y\notin E$. Therefore,  $B_{\beta r}(x') \subseteq B_r(e)\backslash E$. Hence, we have the following property (call it $(*)$): For any $e\in E$ and any $r\in (0, \epsilon]$, there is an $x'\in \cX$ such that $B_{\beta r}(x') \subseteq B_r(e)\backslash E$.
The property $(*)$ is depicted in Figure~\ref{fig: property}.

Let $\alpha = \beta(1-\beta)$.
Take any $x\in \cX$ and $r\in (0, \epsilon]$. We will now show that there exists $x'\in \cX$  such that $B_{\alpha r}(x')\subseteq B_r(x)\backslash E$.

Suppose $x\in E$. Then by the property $(*)$, there exists $x'\in \cX$ such that $B_{\alpha r}(x') \subseteq B_{\beta r}(x') \subseteq B_r(x)\backslash E$.
Suppose on the other hand $x\notin E$. If $B_{\beta r}(x)\cap E = \varnothing$, then  choosing $x'=x$ we have $B_{\alpha r}(x') \subseteq B_{\beta r}(x') \subseteq B_r(x)\backslash E$. If not, then there exists $e\in B_{\beta r}(x)\cap E$. We claim that $B_{(1-\beta)r}(e)\subseteq B_{\beta r}(x)$. Indeed, for any $y\in B_{(1-\beta)r}(e)$ we have
\begin{align*}
    \|y - x\| \leq \|y-e\| + \|e-x\| \leq (1-\beta) r + \beta r = r.
\end{align*}
Since $(1-\beta)r \in (0, \epsilon]$, by the property $(*)$, there exists $x'\in \cX$  such that $B_{\alpha r}(x') = B_{\beta(1-\beta) r}(x') \subseteq B_{(1-\beta)r}(x)\backslash E \subseteq B_r(x)\backslash E$.

\begin{figure}[H]
\centering
\includegraphics[width=0.5\textwidth]{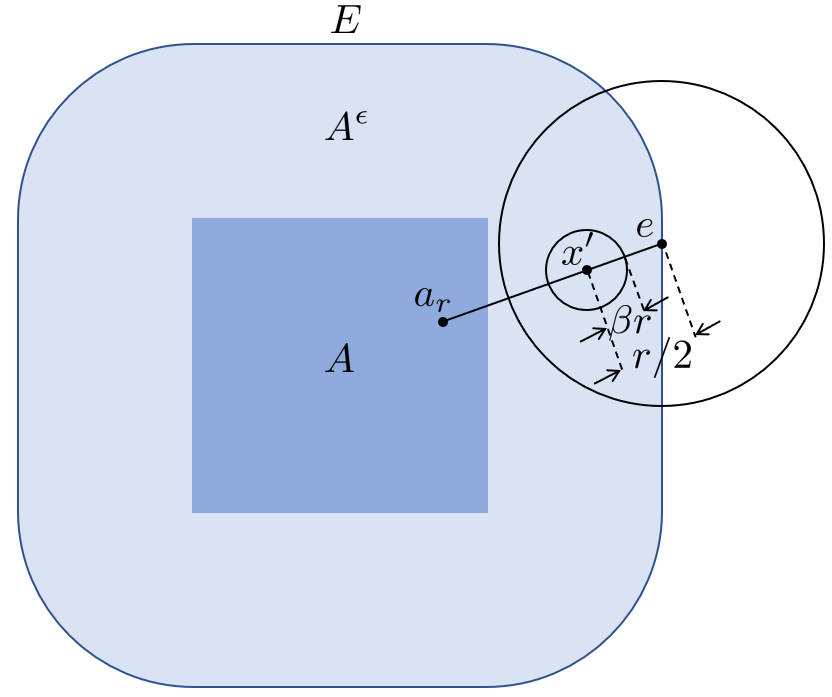}
\caption{A depiction of the property $(*)$ in the proof of Lemma~\ref{lem: A_ep_porous}. $e$ is an arbitrary point in $E = A^{\epsilon}\backslash A^{\epsilon)}$. For some $r\in (0, \epsilon]$, $a_r\in A$ is picked so that $ \|e-a_r\| \in [\epsilon, \epsilon+r/4)$. $x'$ is a point on the line segment joining $a_r$ and $e$ such that $\|x' - e\|  = r/2$.
Then, $B_{\alpha r}(x')\subseteq B_r(e)\backslash E$.}
\label{fig: property}
\end{figure}

\end{proof}

\subsection{Proofs from Section~\ref{sec: well-defined gen loss}}\label{app: well-defined gen loss}

\begin{proof}[Proof of Lemma~\ref{lem: gen loss function measurable}]
Fix $y\in \cY$ and $w\in \cW$. Consider the function $f: D\to \real$ defined as $f(x,x') = -\ell((x', y),w)$, where $D = \{(x,x')\in \cX^2: d(x,x')\leq\epsilon\}$. Define  $f^*: \cX\to \real$ as $f^*(x) \defn \inf_{x'\in B_\epsilon(x)} f(x,x') = -\sup_{x'\in B_\epsilon(x)} \ell((x', y),w)$. By Proposition 7.47 in \cite{BerShr96}, $f^*$ is upper semi-analytic. Therefore,  the worst-case loss function $\sup_{d(x,x')\leq \epsilon} \ell((x', y),w)$ is upper semi-analytic and hence universally measurable. Consequently, $R_{\oplus\epsilon}(\ell, w)$ is well-defined on the measure space $(\cX, \overline{\cB}(\cX))$.
\end{proof}

\begin{proof}[Proof of Lemma~\ref{lem: gen loss function measurable in R^d}]
Since $\ell((\cdot, y), w)$ is Lebesgue measurable, the set $\{x\in \cX:\ell((x,y),w) >t \}^{\oplus\epsilon}$ is Lebesgue measurable. By Lemma~\ref{lem: level set expansion}, all the level sets of the worst-case loss function $\sup_{d(x,x')\leq \epsilon} \ell((x', y),w)$ are Lebesgue measurable. Therefore, 
\begin{align*}
    R_{\oplus\epsilon}(\ell, w) 
    &= \E_{(x,y)\sim \rho}\left[ \sup_{d(x,x')\leq \epsilon} \ell((x', y),w) \right]\\
    &= \E_{y\sim \rho_y}\E_{x\sim \rho_{x|y}}\left[ \sup_{d(x,x')\leq \epsilon} \ell((x', y),w) \right],
\end{align*}
is well-defined.
\end{proof}

\section{Proofs from Section~\ref{sec: W_infty}}

\subsection{Proofs from Section~\ref{sec: measurable selection}}\label{app: measurable selection}

\begin{proof}[Proof of Lemma~\ref{lem: sup measure for Borel sets}]
Let $\mu'\in \cP(\cX)$ be such that $W_\infty(\mu, \mu')\leq \epsilon$. Then there exists a coupling $\lambda\in \Pi(\mu', \mu)$ such that for $(x,x')\sim \lambda$, $d(x,x')\leq \epsilon$ $\lambda$-a.e. Hence,
\begin{align*}
    \mu'(A) = \lambda(A\times\cX) = \lambda(A\times A^{\oplus\epsilon})
    \leq \lambda(\cX\times A^{\oplus\epsilon})= \mu(A^{\oplus\epsilon}).
\end{align*}
Since the choice of $\mu'$ was arbitrary in the set $\{\nu\in \cP(\cX):W_\infty(\mu, \nu)\leq \epsilon\}$, we have,
\begin{align*}
    \sup_{W_\infty(\mu, \mu')\leq\epsilon} \mu'(A) \leq \mu(A^{\oplus\epsilon}).
\end{align*}
Now we show the inequality in the opposite direction. Like in the proof of Lemma~\ref{lem: Borel set expansion}, consider the function $f: D\to \real$ defined as $f(x,x') = -\1\{x'\in A\}$, where $D = \{(x,x')\in \cX^2: d(x,x')\leq\epsilon\}$.
Define  $f^*: \cX\to \real$ as $f^*(x) \defn \inf_{x'\in B_\epsilon(x)} f(x,x')$. As shown in the proof of Lemma~\ref{lem: Borel set expansion}, $f^*(x) = -\1\{x\in A^{\oplus\epsilon}\}$. By Proposition 7.50(a) in \cite{BerShr96}, there exists a measurable function $\phi: \cX\to\cX$ such that $|f^*(x) - f(x, \phi(x))| <\delta$ for any $\delta>0$. Since $f$ and $f^*$ are both $0$-$1$ valued functions, we get $f^*(x) = f(x, \phi(x))$ for all $x\in \cX$ by choosing $\delta=1/2$. Moreover, by Proposition 7.50(a) in \cite{BerShr96}, $Gr(\phi)\subseteq D$ i.e., $d(x,\phi(x))\leq \epsilon$ for all $x\in \cX$. Therefore, 
\begin{align*}
    \sup_{W_\infty(\mu, \mu')\leq\epsilon}  \mu'(A)
    \geq \phi_{\sharp \mu}(A)
    = \mu(\phi^{-1}(A))
    = \mu(A^{\oplus\epsilon}).
\end{align*}

Hence, $ \sup_{W_\infty(\mu, \mu')\leq\epsilon} \mu'(A) = \phi_{\sharp \mu}(A) = \mu(A^{\oplus\epsilon})$ for any  set $A\in \cB(\cX)$.
\end{proof}

\begin{proof}[Proof of Lemma~\ref{lem: sup expectation for Polish}]
Let $\mu'\in \overline{\cP}(\cX)$ be such that $W_\infty(\mu, \mu')\leq\epsilon$. Then there exists $\lambda\in \Pi(\mu', \mu)$ such that $\lambda(\{(x,x')\in \cX^2: d(x,x')>\epsilon\}) = 0$. Then,
\begin{align*}
    \E_{x\sim \mu'}[\phi(x)]
    &= \E_{(x,x')\sim \lambda} [\phi(x)]\\
    &= \E_{(x,x')\sim \lambda} \left[ \sup_{x\in B_\epsilon(x')} \phi(x')\right]\\
    &= \E_{x'\sim \mu} \left[ \sup_{x\in B_\epsilon(x')} \phi(x')\right].
\end{align*}
Since the above inequality is true for any  $\mu'\in \overline{\cP}(\cX)$  satisfying $W_\infty(\mu, \mu')\leq\epsilon$, we have,
\begin{align*}
    \sup_{W_\infty(\mu, \mu')\leq\epsilon} \E_{x\sim \mu'}[\phi(x)] \leq \E_{x'\sim \mu} \left[ \sup_{x\in B_\epsilon(x')} \phi(x')\right].
\end{align*}

Now we will show the inequality in the opposite direction.  
Consider the function $f: D\to \real$ defined as $f(x,x') = -\phi(x')$, where $D = \{(x,x')\in \cX^2: d(x,x')\leq\epsilon\}$. Define  $f^*: \cX\to \real$ as $f^*(x) \defn \inf_{x'\in B_\epsilon(x)} f(x,x') = -\sup_{x'\in B_\epsilon(x)} \phi(x')$. Choose a $\delta>0$. By Proposition 7.50(a) in \cite{BerShr96},  there exists a universally measurable function $m_\delta: \cX\to \cX$ such that $|f^*(x) - f(x, m_\delta(x))|\leq\delta$ and $d(x, m_\delta(x)\leq\epsilon)$ for all $x\in \cX$. Hence,

\begin{align*}
    \E_{x\sim \mu} \left[ \sup_{d(x,x')\leq \epsilon} \phi(x') \right]
    &= \E_{x\sim \mu} [-f^*(x)]\\
    &\leq \E_{x\sim \mu} [-f(x, m_\delta(x))] + \delta\\
    &= \E_{x\sim \mu} [\phi(m_\delta(x))] + \delta\\
    &= \E_{x\sim {m_{\delta}}_{\sharp\mu}} [\phi(x)] + \delta\\
    &\leq \sup_{W_\infty(\mu, \mu')\leq\epsilon} \E_{x\sim \mu'}[\phi(x)] + \delta,
\end{align*}
where the last inequality follows because $W_\infty(\mu, {m_{\delta}}_{\sharp\mu})\leq\epsilon$ because $d(x, m_\delta(x)\leq\epsilon)$ for all $x\in \cX$. Taking $\delta\to 0$, we get the following inequality. 
\begin{align*}
    \E_{x\sim \mu} \left[ \sup_{d(x,x')\leq \epsilon} \phi(x') \right] \leq \sup_{W_\infty(\mu, \mu')\leq\epsilon} \E_{x\sim \mu'}[\phi(x)].
\end{align*}

Combining the above inequality with the reverse inequality shown previously, we obtain \eqref{eq: sup expectation for Polish}.

Suppose the function $\phi$ is upper semi-continuous. Then $f$ is lower semi-continuous. Hence, for every $x\in \cX$, there exists $x^*$ in the compact set $B_\epsilon(x)$ such that $\inf_{x'\in B_\epsilon(x)} f(x,x') = f(x,x^*)$.
By Proposition 7.50(b), there exists a universally measurable function $m: \cX\to \cX$ such that $f^*(x) = f(x, m(x))$ for all $x\in \cX$. Hence, we have
\begin{align*}
    \sup_{W_\infty(\mu, \mu')\leq\epsilon} \E_{x\sim \mu'}[\phi(x)]
    =  \E_{x\sim \mu} \left[ \sup_{d(x,x')\leq \epsilon} \phi(x') \right]
    =  \E_{x\sim \mu} \left[  \phi(m(x)) \right]
    = \E_{x\sim m_{\sharp\mu}} \left[  \phi(x) \right].
\end{align*}
Therefore, $\mu^*\defn m_{\sharp\mu}$ attains the supremum on the left side of the above equation.

\end{proof}

\subsection{Proofs from Section~\ref{sec: capacities}}\label{app: capacities}

\begin{proof}[Proof of Lemma~\ref{lem: set expansion is 2-alternating capacity}]
The following properties of $v$ are trivially true: $v(\phi) = 0$, $v(\cX) = 1$ and 
$v(A)\leq v(B)$ for $A\subseteq B$.

Consider a sequence of sets $(A_n)$ in $\cX$ such that $A_k\subseteq A_{k+1}$ for $k\in \N$. Let $A = \cup_n A_n$. That is, $A_n \uparrow A$.  Then by Lemma~\ref{lem: Minkowski sum of union} we have, $A^{\oplus\epsilon} = \cup_n A_n^{\oplus\epsilon}$. Hence, $A_n^{\oplus\epsilon} \uparrow A^{\oplus\epsilon}$ and by the continuity of measure, $v(A_n) = \mu(A_n^{\oplus\epsilon})\uparrow \mu(A^{\oplus\epsilon})= v(A)$.

Consider a sequence of closed sets $(F_n)$ in $\cX$ such that $F_k\supseteq F_{k+1}$ for $k\in \bN$. Let $F = \cap_n F_n$. That is, $F_n \downarrow F$. 
By Lemma~\ref{lem: convergence of closed sets}, $F_n^{\oplus\epsilon} \downarrow F^{\oplus\epsilon}$. Hence, by the continuity of measure, we have $v(F_n)= \mu(F_n^{\oplus\epsilon})\downarrow \mu(F^{\oplus\epsilon})= v(F)$.

For any two sets $A, B\in \cL(\cX)$, 
\begin{align*}
    v(A\cup B)
    &= \mu((A\cup B)^{\oplus\epsilon})\\
    &\stackrel{(i)}{=} \mu(A^{\oplus\epsilon} \cup B^{\oplus\epsilon})\\
    &= \mu(A^{\oplus\epsilon}) + \mu(B^{\oplus\epsilon}) - \mu(A^{\oplus\epsilon}\cap B^{\oplus\epsilon})\\
    &\stackrel{(ii)}{\leq} \mu(A^{\oplus\epsilon}) + \mu(B^{\oplus\epsilon}) - \mu((A\cap B)^{\oplus\epsilon})\\
    &= v(A)+v(B)-v(A\cap B),
\end{align*}
where $(i)$ and $(ii)$ follow from Lemma~\ref{lem: Minkowski sum of union}.
Hence, $v$ is a $2$-alternating capacity.
\end{proof}

\begin{proof}[Proof of Lemma~\ref{lem: sup measure for measurable sets}]
Let $\mu'\in \cP(\cX)$ be such that $W_\infty(\mu, \mu')\leq \epsilon$. Then there exists a coupling $\gamma\in \Pi(\mu', \mu)$ such that for $(x,x')\sim \gamma$, $d(x,x')\leq \epsilon$ $\gamma$-a.e. Hence,
\begin{align*}
    \mu'(A) = \gamma(A\times\cX) = \gamma(A\times A^{\oplus\epsilon})
    \leq \gamma(\cX\times A^{\oplus\epsilon})= \mu(A^{\oplus\epsilon}).
\end{align*}
Since the choice of $\mu'$ was arbitrary in the set $\{\nu\in \cP(\cX):W_\infty(\mu, \nu)\leq \epsilon\}$, we have,
\begin{align*}
    \sup_{W_\infty(\mu, \mu')\leq\epsilon} \mu'(A) \leq \mu(A^{\oplus\epsilon}).
\end{align*}
We will now show the inequality in the reverse direction.
By Lemma~\ref{lem: set expansion is 2-alternating capacity}, $A \mapsto \mu(A^{\oplus\epsilon})$ is a $2$-alternating capacity. 
Hence by Lemma 2.5 in \cite{HubStr73}, for any Lebesgue measurable $A\subseteq \cX$, there exists a $\nu\in \cP(\cX)$ such that $\nu(A) = \mu(A^{\oplus\epsilon})$ and  $\nu(B)\leq \mu(B^{\oplus\epsilon})$ for all Lebesgue measurable $B\subseteq \cX$. For such a $\nu$, it is clear that $W_\infty(\mu, \nu)\leq\epsilon$. Hence,
\begin{align*}
    \sup_{W_\infty(\mu, \mu')\leq\epsilon} \mu'(A)
    \geq \nu(A) = \mu(A^{\oplus\epsilon}).
\end{align*}
Hence, $\sup_{W_\infty(\mu, \mu')\leq\epsilon} \mu'(A) = \nu(A) = \mu(A^{\oplus\epsilon}).$
\end{proof}

\section{Proofs from Section~\ref{sec: strassen}}\label{app: strassen}

\subsection{Proofs from Section~\ref{sec: unbalanced OT}}\label{app: unbalanced OT}

We first prove a discrete version of Theorem~\ref{thm: generalized strassen} on a finite space. 

\begin{lemma}\label{lem: discrete strassen}
Let $\cX_n = \{x_1, \ldots, x_n\}\subseteq \cX$. 
Let $\bp = (p_i)_{i=1}^n, \bq = (q_i)_{i=1}^n$ be such that $p_i, q_i\geq 0$ for $i\in [n]$ and $\sum_i p_i \leq \sum_i q_i$. Let $\epsilon>0$. 
For $A\subseteq\cX_n$, let  $A^\epsilon := \{x\in \cX_n: d(x, x')\leq \epsilon \text{, for some } x'\in A\}$. 
For $A\subseteq\cX_n$, let $\bp(A) = \sum_{i: x_i\in A} p_i$ and $\bq(A) = \sum_{i: x_i\in A} q_i$. 
For $i,j\in [n]$, let $c_{ij} = \mathds{1}\{d(x_i, x_j)>2\epsilon\}$.
Then,
\begin{align}\label{eq: discrete strassen}
    \max_{A\subseteq \cX_n} \bp(A)-\bq(A^{2\epsilon})
    =\min_{\substack{x_{ij}\geq 0 \\ \sum_j x_{ij} = p_i\\ \sum_i x_{ij} \leq q_j}}
    \sum_{i,j} c_{ij}  x_{ij} .    
\end{align}
\end{lemma}
\begin{proof}

For $i,j\in [n]$, define $d_{ij} := 1-c_{ij}$. Then,
\begin{align}\label{eq: lem discrete max}
    \min_{\substack{x_{ij}\geq 0 \\ \sum_j x_{ij} = p_i\\ \sum_i x_{ij} \leq q_j}}
    \sum_{i,j} c_{ij}  x_{ij} 
    = 
    \sum_i p_i - 
    \max_{\substack{x_{ij}\geq 0 \\ \sum_j x_{ij} = p_i\\ \sum_i x_{ij} \leq q_j}}
    \sum_{i,j} d_{ij}  x_{ij} 
\end{align}
Consider the following modification to the linear program on the right hand side of \eqref{eq: lem discrete max}, where the constraint $\sum_j  x_{ij} = p_i$ is replaced by $\sum_j  x_{ij} \leq p_i$.

\begin{align}\label{eq: lin prog}
    \max_{\substack{x_{ij}\geq 0 \\ \sum_j x_{ij} \leq p_i\\ \sum_i x_{ij} \leq q_j}}
    \sum_{i,j} d_{ij}  x_{ij}   .  
\end{align}
We will show that the above linear program is equivalent to the linear program on the right hand side of \eqref{eq: lem discrete max}. 
Since the above linear program is bounded and feasible, it admits a solution. Let $\{x^*_{ij}\}_{i,j\in [n]}$ be the solution to \eqref{eq: lin prog}. 
Suppose there exists $m\in [n]$ such that $\sum_j x^*_{mj}< p_m$.
Let $s = p_m - \left(\sum_j x^*_{mj}\right) >0 $.
For $j\in [n]$, define $s_j := q_j - \sum_i x^*_{ij}$. Then,
\begin{align*}
	\sum_j s_j 
	&= \sum_j q_j - \sum_{i,j}  x_{ij} \\
	&\geq \sum_i p_i - \left( \left(  \sum_{i\neq m} p_i  \right) + p_m - s\right)\\
	&= s.
\end{align*}
Therefore, $\sum_j s_j \geq s$.
Let $k$ be the largest integer for which  $\sum_{j=1}^k s_j < s$. Define,
\begin{align}
    y_{ij} = 
    \begin{cases}
    x^*_{ij} & i\neq m,\\
    x^*_{mj}+s_j & i= m, j\leq k,\\
    x^*_{mk}+ s - \sum_{j=1}^k s_j & i= m, j=k+1,\\
    x^*_{mj}  & i= m, j\geq k+1.
    \end{cases}
\end{align}
By the above definition we have,
\begin{align*}
    \sum_j y_{ij} &= 
    \begin{cases}
    \sum_j x^*_{ij} & i\neq m,\\
    \sum_j x^*_{ij} + s & i = m.\\
    \end{cases}	\\
    \sum_i y_{ij} &= 
    \begin{cases}
    \sum_i x^*_{mj}+s_j &  j\leq k,\\
    \sum_i x^*_{mk}+ s - \sum_{j=1}^k s_j &  j=k+1,\\
    \sum_i x^*_{mj}  &  j\geq k+1.
    \end{cases}	    
\end{align*}
Combining the above with the definitions of $k, s$ and $\{s_j\}_{j\in [n]}$, we see that $ \sum_j y_{ij}\leq p_i$ and $\sum_i y_{ij}\leq q_j$.
Moreover, $y_{ij}\geq x_{ij}$ for all $i,j\in [n]$. Hence, $\sum_{ij} d_{ij}y_{ij} \geq \sum_{ij} d_{ij}x_{ij}$. 
Therefore, any solution $\{x^*_{ij}\}_{i,j\in [n]}$ for which there exists $m\in [n]$ such that $\sum_j x^*_{mj}< p_m$, can be improved to a solution
$\{y_{ij}\}_{i,j\in [n]}$ for which $\sum_j y_{mj} = p_m$. 
Hence,
\begin{align}\label{eq: lem dicrete max 2}
    \max_{\substack{x_{ij}\geq 0 \\ \sum_j x_{ij} = p_i\\ \sum_i x_{ij} \leq q_j}}
    \sum_{i,j} d_{ij}  x_{ij} 
    =
    \max_{\substack{x_{ij}\geq 0 \\ \sum_j x_{ij} \leq p_i\\ \sum_i x_{ij} \leq q_j}}
    \sum_{i,j} d_{ij}  x_{ij} .    
\end{align}
Since the maximization in \eqref{eq: lem dicrete max 2} is a linear program in canonical form, we employ the strong duality theorem (for a reference, see Chapter 6 in \cite{MatGar07}) to get the following.
\begin{align}\label{eq: lem dual LP}
    \max_{\substack{x_{ij}\geq 0 \\ \sum_j x_{ij} \leq p_i\\ \sum_i x_{ij} \leq q_j}}
    \sum_{i,j} d_{ij}  x_{ij}
    =
    \min_{\substack{u_i, v_i\geq 0 \\ u_i + v_j \geq d_{ij}}}
    \sum_i (p_i u_i + q_i v_i).
\end{align}
Since $d_{ij}\in \{0,1\}$, we may assume $u_i, v_i \leq 1$ for the minimization in \eqref{eq: lem dual LP} without violating other constraints because any decrease of $u_i, v_i$ down to $1$ will only decrease the value of $\sum_i (p_i u_i + q_i v_i)$, which we seek to minimize.
Defining $w_i := 1-u_i$,  we have the following from \eqref{eq: lem discrete max} and \eqref{eq: lem dual LP}.
\begin{align}\label{eq: lem dual LP 2}
    \min_{\substack{x_{ij}\geq 0 \\ \sum_j x_{ij} = p_i\\ \sum_i x_{ij} \leq q_j}}
    \sum_{i,j} c_{ij}  x_{ij} 
    =
    \max_{\substack{w_i, v_i\in [0,1]\\ w_i - v_j \leq c_{ij}}}
     \sum_i ( p_iw_i - q_iv_i) .
\end{align}
The optimal $w_i^*, v_i^*$ that achieve the maximum in \eqref{eq: lem dual LP 2} must  lie at one of the vertices of the polyhedron supported by the hyperplanes, $w_i = 0, w_i = 1, v_i = 0, v_i = 1$ and $w_i-v_j = c_{ij}$. Hence, $w_i^*, v_i^* \in \{0,1\}$.
Moreover  if $c_{ij} = 0$ and $w_i^* = 1$ for some $i,j\in [n]$, then  $v_j^* = 1$. On the other hand if $c_{ij}=1$,  then $v_j^*$ can be set to $0$ without violating other constraints and without decreasing the maximization objective. Therefore, setting $A := \{x_i\in \cX_n : w_i^* = 1\}$, we see that the maximum in \eqref{eq: lem dual LP 2} equals the maximum in  \eqref{eq: discrete strassen}.
\end{proof}

\begin{proof}[Proof of Theorem~\ref{thm: generalized strassen}]

Let $(\gamma_n)_{n=1}^\infty$ be a non-negative, monotonically decreasing sequence converging to $0$. 
Let $(x_n)_{n=1}^\infty$ be a dense sequence in $\mathcal{X}$. 
Define a  function $f: \cX \to \{x_n\}_{n=1}^\infty$ such that 
$f(x) = x_k$ for the least integer $k$ with $d(x,x_k) < \gamma_n$.
Let $H_n = \{x_1, \ldots, x_n\}$.
Let $s_n$ be the least positive integer such that,
\begin{align}
	\mu(f^{-1}(H_{s_n-1})) > \mu(\cX) - \gamma_n, \label{eq: mu_n last}\\
	\nu(f^{-1}(H_{s_n-1})) > \nu(\cX) - \gamma_n. \label{eq: nu_n last}
\end{align}

Given $n$, construct a discrete measure $\mu_n$ supported on the finite set $H_{s_n}$ such that 
$\mu_n(x_k) := \mu(f^{-1}(x_k))$ for $k\in [s_n-1]$ and $\mu_n(\cX) = \mu(\cX)$.
Similarly, construct $\nu_n$ supported on $H_{s_n}$ such that 
$\nu_n(x_k) := \nu(f_n^{-1}(x_k))$ for $k\in [s_n-1]$ and $\nu_n(\cX) = \nu(\cX)$.

Let $A\in \cB(\cX)$. We have,
\begin{align}\label{eq: discrete measure upper bound}
    \mu_n(A)
    &\stackrel{(i)}{=}  \mu_n(A \cap H_{s_n})\nonumber\\
    &\stackrel{(ii)}{<}  \mu_n(A \cap H_{s_n-1}) + \gamma_n\nonumber\\
    &\stackrel{(iii)}{=}  \mu(f^{-1}(A\cap H_{s_n-1})) + \gamma_n\nonumber\\
    &\stackrel{(iv)}{\leq} \mu(A^{\gamma_n}) + \gamma_n,
\end{align}
where $(i)$ follows from the fact that $\mu_n$ is supported on $H_{s_n}$, $(ii)$ follows from \eqref{eq: mu_n last}, $(iii)$ follows from the definition of $\mu_n$ and $(iv)$ follows because of the following: For any $y\in A\cap H_{s_n-1}$,  $f^{-1}(y) \subseteq \{x\in \cX: d(x, y) < \gamma_n\}\subseteq A^{\gamma_n}$. Hence, $f^{-1}(A\cap H_{s_n-1}) \subseteq A^{\gamma_n}$.
Applying \eqref{eq: discrete measure upper bound}, with $A^c$ instead of $A$, we have the following.
\begin{align}\label{eq: discrete measure UB and LB}
    \mu(A^{-\gamma_n}) - \gamma_n \leq \mu_n(A) 
    \leq \mu(A^{\gamma_n}) + \gamma_n.
\end{align}

Letting $n\to \infty$ in \eqref{eq: discrete measure UB and LB}  and using Lemma~\ref{lem: closed set}, we get that $\limsup_n \mu_n(A) \leq \mu(A)$ for all closed subsets $A$ of $\cX$. Hence, by applying the Portmanteau theorem (Theorem 2.1 in \cite{Bil99}), we conclude that the sequence of measures $(\mu_n)_{n=1}^\infty$ converges weakly to $\mu$. Similarly, $\nu_n \to \nu$ weakly.

For any fixed $n$, we apply Lemma~\ref{lem: discrete strassen} 
to the measures $\mu_n, \nu_n$
on the finite space $H_{s_n}$ to get the following.
\begin{align}\label{eq: discrete coupling}
    \max_{A\subseteq H_{s_n}} \mu_n(A)-\nu_n(A^{{2\epsilon}+4\gamma_n})
    =\min_{\substack{x_{ij}\geq 0 \\ \sum_j x_{ij} = \mu_n(x_i)\\ \sum_i x_{ij} \leq \nu_n(x_j)}}
    \sum_{i,j} x_{ij}\mathds{1}\{d(x_i, x_j)>2\epsilon+4\gamma_n\},    
\end{align}
where the indices $i,j$ run over $[s_n]$.
We have that $\mu_n(\cX) = \mu(\cX)\leq \nu(\cX) = \nu(\cX)$.
Define a coupling $\pi_n\in \Pi(\mu_n, \nu_n)$ supported on $H_{s_n}\times H_{s_n}$ using the optimal solution $\{x_{ij}\}_{i,j\in [s_n]}$ to the minimization in \eqref{eq: discrete coupling} by setting $\pi_n(i, j) = x_{ij}^*$.
Let $T_n\subseteq H_{s_n}$ be the set that achieves the maximum in \eqref{eq: discrete coupling}.

We will now construct a candidate coupling for the infimum in \eqref{eq: generalized strassen}. Since $\mu, \nu$ are finite  measures on a Polish space, they are tight (see for example, Theorem 1.3 in \cite{Bil99}). Hence, given a $\delta>0$, there exists a compact set $K\subseteq \cX$ such that $\min\{\mu(K^c), \nu(K^c)\}< \delta/3$. Since $\mu_n$ and $\nu_n$ converge weakly to $\mu$ and $\nu$ respectively, choose $N$ large enough so that $\min\{\mu_n(K^c), \nu_n(K^c)\} < \delta/2$ for all $n\geq N$. 
Let $\nu'_n$ be the second marginal of the coupling $\pi_n$. Then, $\nu'_n\preceq \nu_n$. By union bound, we have the following.
\begin{align}
    \pi_n((K\times K)^c) \leq \mu_n(K^c) + \nu'_n(K^c) \leq \mu_n(K^c) + \nu_n(K^c) < \delta.
\end{align}
Hence, the sequence $(\pi_n)_{n\geq N}$ is uniformly tight. Hence, by Prokhorov's theorem (for reference, see Theorem 5.1 in \cite{Bil99}), there is a subsequence $(\pi_{n_k})$ of $(\pi_n)_{n\geq N}$ that converges weakly to some measure $\pi^*\in \cM(\cX\times\cX)$. Moreover, $\pi^*\in \Pi(\mu, \nu)$ by virtue of the constraints imposed on the converging subsequence of $(\pi_n)_{n\geq N}$.

Let $\Phi = \sup_{A\in \cB(\cX)} \mu(A) - \nu(A^{2\epsilon})$ and $\Psi = \cT_{c_\epsilon}(\mu, \nu)$. For any $n$ we have,
\begin{align}
    \pi_n(d(x_i, x_j)>2\epsilon+4\gamma_n)
    &\stackrel{(i)}{=} \mu_n(T_n) - \nu_n(T^{2\epsilon+4\gamma_n}_n)\nonumber\\
    &\stackrel{(ii)}{\leq} (\mu(T^{\gamma_n}_n)+\gamma_n) - (\nu((T^{2\epsilon+4\gamma_n}_n)^{-\gamma_n}) - \gamma_n)\nonumber\\
    &\stackrel{(iii)}{\leq} \mu(T^{\gamma_n}_n)- \nu((T^{2\epsilon+4\gamma_n-\gamma_n-\gamma_n/2}_n) + 2\gamma_n\nonumber\\ 
    &\leq \mu(T^{2\gamma_n}_n)- \nu((T^{2\epsilon+2\gamma_n}_n) + 2\gamma_n\nonumber\\     
    &\stackrel{(iv)}{\leq} \Phi + 2\gamma_n \label{eq: z1},   
\end{align}
where $(i)$ follows from the definition of $\pi_n$ and  $T_n$, $(ii)$ follows from \eqref{eq: discrete measure UB and LB}, $(iii)$ follows from Lemma~\ref{lem: set expansion property} and $(iv)$ follows from the definition of $\Phi$. Further,
\begin{align}
    \Psi
    &= \inf_{\pi\in\Pi(\mu, \nu)} \pi[d(x,x')>2\epsilon]\nonumber\\
    &\stackrel{(i)}{\leq} \pi^*[d(x,x')>2\epsilon]\nonumber\\
    &\stackrel{(ii)}{\leq} \liminf_{n_k} \pi_{n_k}[d(x,x')>2\epsilon]\nonumber\\
    &\leq \limsup_n \pi_n[d(x,x')\geq 2\epsilon]\nonumber\\    
    &\stackrel{(iii)}{\leq}  \Phi\label{eq: z2},
\end{align}
where $(i)$ follows because $\pi^*\in \Pi(\mu, \nu)$, $(ii)$ follows from Portmanteau's theorem because  $(\pi_{n_k})$ that converges to $\pi^*$ and the set $\{(x,x')\in \cX^2: d(x,x')>2\epsilon\}$ is an open set, and $(iii)$ follows by taking $n\to\infty$ in \eqref{eq: z1}.

To show the inequality $\Phi\leq \Psi$,  consider a sequence of measures $(\lambda_n)_{n=1}^\infty$ such that $\lambda_n\in \Pi(\mu, \nu)$ and $\lim_n \lambda_n[d(x,x')>\epsilon] = \Psi$.
For any  $A\in \cB(\cX)$,
\begin{align*}
    \mu(A) 
    &= \lambda_n[x\in A, x'\in A^\epsilon] + \lambda_n[x\in A, x'\notin A^\epsilon]\\
    &\leq \nu(A^\epsilon) + \lambda_n[d(x,x')>\epsilon].
\end{align*}
Letting $n\to\infty$, we have $\mu(A)-\nu(A^\epsilon)\leq \Psi$ for all $A\in \cB(\cX)$. Hence, $\Phi\leq\Psi$.
Combining this with \eqref{eq: z2}, we conclude $\Phi = \Psi$.
\end{proof}

\subsection{Proof of Section~\ref{sec: opt adv risk unequal priors}}\label{app: opt adv risk unequal priors}

\begin{proof}[Proof of Lemma~\ref{lem: pydijog general}]
We have,
\begin{align*}
    \sup_{A\in \cB(\cX)} \mu(A) - \nu(A^{2\epsilon})
    &\stackrel{(i)}{=} \sup_{A\ \text{closed}} \mu(A) - \nu(A^{2\epsilon})\\
    &\stackrel{(ii)}{\leq} \sup_{A\ \text{closed}} \mu((A^{\oplus\epsilon})^{\ominus\epsilon}) - \mu((A^{\oplus\epsilon})^{\oplus\epsilon})\\
    &\stackrel{(iii)}{\leq} \sup_{A\in \cB(\cX)} \mu(A^{\ominus\epsilon}) - \nu(A^{\oplus\epsilon}),
\end{align*}
where $(i)$ follows because we may assume that the supremum of $\mu(A) - \nu(A^{2\epsilon})$ is achieved by a closed set. Indeed, 
$\mu(\overline{A}) - \nu(\overline{A}^{2\epsilon}) \geq \mu(A) - \nu(A^{2\epsilon})$ because $\overline{A}\supseteq A$ and $\overline{A}^{2\epsilon} = A^{2\epsilon}$. $(ii)$ follows from the following two facts: 1) $A\subseteq A^{\oplus\epsilon})^{\ominus\epsilon}$ (see Lemma 3.3 in \cite{PydJog21}), and 2) $A^\epsilon = A^{\oplus\epsilon}$ for closed sets $A$ (see Lemma 3.2 in \cite{PydJog21}).
$(iii)$ follows from Lemma~\ref{lem: Borel set expansion} because $\mu, \nu\in \overline{\cP}(\cX)$ and  $A^{\ominus\epsilon}, A^{\oplus\epsilon}\in \overline{\cB}(\cX)$ whenever $A\in \cB(\cX)$.

Now, we show that the above inequality also holds in the opposite direction.
Let $\mu' = \mu/t \in \overline{\cP}(\cX)$ for some fixed $t>0$. 
For $x,y\in \cX$, define the cost function $c(x,y) = \1\{d(x,y)>2\epsilon\}$.
For any $\nu'\in \overline{\cP}(\cX)$, we have the following from Kantorovich duality theorem.
\begin{align*}
     D_\epsilon(\mu', \nu')
    = \sup_{\phi(x)+\psi(y)\leq c(x,y)} \int \phi d\mu' + \int \psi d\nu',
\end{align*}
For any $A\in \cB(\cX)$, define $\phi'(x) = \1\{x\in A^{\ominus\epsilon}\}$ and $\psi'(y) = -\1\{y\in A^{\oplus \epsilon}\}$.
We will now show that $\phi'(x)+\psi'(y)\leq c(x,y)$.
If $x,y$ are such that $c(x,y) = 1$, the inequality holds trivially.
Suppose on the other hand, $x,y$ are such that $c(x,y) = 0$. Then $d(x,y)\leq 2\epsilon$. Hence, for any $x\in A^{\ominus\epsilon}$, we have $y\in (A^{\ominus\epsilon})^{\oplus 2\epsilon}= ((A^{\ominus\epsilon})^{\oplus \epsilon})^{\oplus\epsilon}\subseteq A^{\oplus\epsilon}$ (the set inclusion here follows from Lemma 3.3 in \cite{PydJog21}). Therefore,
\begin{align*}
    \phi'(x)+\psi'(y)
    = \1\{x\in A^{\ominus\epsilon}\} - \1\{y\in A^{\oplus \epsilon}\} = 0 = c(x,y).
\end{align*}
Hence,
\begin{align*}
    D_\epsilon(\mu', \nu')
    \geq \int \phi' d\mu' + \int \psi' d\nu'
    = \mu'(A^{\ominus\epsilon}) - \nu'(A^{\oplus \epsilon}).
\end{align*}
Now,
\begin{align*}
    \sup_{A\in \cB(\cX)} \mu(A) - \nu(A^{2\epsilon})
    &\stackrel{(*)}{=} t \inf_{\substack{\nu'\in \overline{\cP}(\cX)\\ \nu'\preceq \nu/t}} D_\epsilon(\mu', \nu')\\
    &\geq t \inf_{\substack{\nu'\in \overline{\cP}(\cX)\\ \nu'\preceq \nu/t}}
    \mu'(A^{\ominus\epsilon}) - \nu'(A^{\oplus \epsilon})\\
    &= t\mu'(A^{\ominus\epsilon}) - t\sup_{\substack{\nu'\in \overline{\cP}(\cX)\\ \nu'\preceq \nu/t}} \nu'(A^{\oplus \epsilon})\\
    &\geq \mu(A^{\ominus\epsilon}) - t\frac{\nu}{t}(A^{\oplus \epsilon})\\
    &= \mu(A^{\ominus\epsilon}) - \nu(A^{\oplus \epsilon}),
\end{align*}
where $(*)$ follows from Theorem~\ref{thm: generalized strassen}. Since the above inequality is valid for any $A\in \cB(\cX)$, we get the following.
\begin{align*}
    \sup_{A\in \cB(\cX)} \mu(A) - \nu(A^{2\epsilon})
    \geq \sup_{A\in \cB(\cX)} \mu(A^{\ominus\epsilon}) - \nu(A^{\oplus\epsilon}).
\end{align*}

\end{proof}

\section{Proofs from Section~\ref{sec: minimax}}\label{app: minimax}

\subsection{Proofs from Section~\ref{sec: minimax in Rd}}\label{app: minimax in Rd}

\begin{proof}[Proof of Theorem~\ref{thm: minimax theorem in R^d}]
By Lemma~\ref{lem: set expansion is 2-alternating capacity}, the set-valued maps $A \mapsto p_0(A^{\oplus\epsilon})$ and $A^c \mapsto p_1((A^c)^{\oplus\epsilon})$ are $2$-alternating capacities. Hence, the existence of $A^*\in \cL(\cX)$ that attains the infimum on the right  in \eqref{eq: min-max theorem in Rd} follows from Lemma 3.1 in \cite{HubStr73} and the equality $R_{\oplus\epsilon}(\ell_{0/1}, A) = R_{\Gamma_\epsilon}(\ell_{0/1}, A)$ proved in Theorem~\ref{thm: Adv risk in Rd}.
By Theorem 4.1 in \cite{HubStr73}, there exist $q_0, q_1\in \cP(\cX)$ such that $W_\infty(p_i, q_i)\leq \epsilon$ for $i=0,1$ and,
\begin{align*}
    \inf_{A\in \cL(\cX)} \sup_{W_\infty(p_0, p_0'), W_\infty(p_1, p_1')\leq \epsilon} r(A, p_0, p_1)
    = \inf_{A\in \cL(\cX)} r(A, q_0, q_1).
\end{align*}
Hence,
\begin{align*}
    \inf_{A\in \cL(\cX)} \sup_{W_\infty(p_0, p_0'), W_\infty(p_1, p_1')\leq \epsilon} r(A, p_0, p_1)
    & = \inf_{A\in \cA} r(A, q_0, q_1) \\
    &
    \leq \sup_{W_\infty(p_0, p_0'), W_\infty(p_1, p_1')\leq \epsilon} \inf_{A\in \cL(\cX)}  r(A, p_0, p_1).
\end{align*}
The desired result follows from combining the above inequality with the max-min inequality~\eqref{eq: max-min inequality}. Clearly, $q_0 = p_0^*$ and $q_1 = p_1^*$. 
\end{proof}

\subsection{Proofs from Section~\ref{sec: minimax in Polish space}}\label{app: minimax in Polish space}

\begin{lemma}[Max-min Inequality]\label{lem: max-min inequality}
Let $p_0, p_1\in \cP(\cX)$ and let $\epsilon\geq 0$.
For $T>0$, define $r: \cB(\cX) \times \cP(\cX)\times\cP(\cX) \to [0,1]$ as in \eqref{eq: payoff function}.
Then,
\begin{align}
    \sup_{\substack{W_\infty(p_0, p_0')\leq \epsilon\\W_\infty(p_1, p_1')\leq \epsilon}}
    \inf_{A\in \cB(\cX)}  r(A, p_0', p_1')
    \leq 
    \inf_{A\in \cB(\cX)} 
    \sup_{\substack{W_\infty(p_0, p_0')\leq \epsilon\\W_\infty(p_1, p_1')\leq \epsilon}}  r(A, p_0', p_1').
\end{align}
\end{lemma}

\begin{proof}
For any $A\in \cB(\cX)$ and $p_0', p_1'$ such that $W_\infty(p_i, p_i')\leq\epsilon$ ($i=0,1$), we have
\begin{align*}
    \inf_{A\in \cB(\cX)} r(A, p_0', p_1') 
    \leq r(A, p_0', p_1').
\end{align*}
Taking supremum over $p_0'$ and $p_1'$ such that $W_\infty(p_i, p_i')\leq \epsilon$ for $i\in \{0,1\}$ on both sides of the above inequality, we get the following for any $A\in \cB(\cX)$.
\begin{align*}
    \sup_{W_\infty(p_0, p_0'), W_\infty(p_1, p_1')\leq \epsilon}
    \inf_{A\in \cB(\cX)}  r(A, p_0', p_1')
    \leq \sup_{W_\infty(p_0, p_0'), W_\infty(p_1, p_1')\leq \epsilon}  r(A, p_0', p_1').
\end{align*}
Since the above inequality holds for any $A\in \cB(\cX)$, we have,
\begin{align*}
    \sup_{W_\infty(p_0, p_0'), W_\infty(p_1, p_1')\leq \epsilon} \inf_{A\in \cB(\cX)}  r(A, p_0', p_1')
    \leq 
    \inf_{A\in \cB(\cX)} \sup_{W_\infty(p_0, p_0'), W_\infty(p_1, p_1')\leq \epsilon}  r(A, p_0', p_1').    
\end{align*}
\end{proof}

\begin{proof}[Proof of Theorem~\ref{thm: D_ep as shortest D_TV}]
Consider any $\mu'$ and $\nu'$ such that $W_\infty(\mu, \mu')\leq \epsilon$ and $W_\infty(\nu, \nu')\leq \epsilon$. Then there exist $\gamma_\mu\in \Pi(\mu, \mu')$ and $\gamma_\nu\in \Pi(\nu, \nu')$ such that
\begin{align*}
    \prob_{(x,x')\sim \gamma_\mu} (d(x,x')>\epsilon) = 0,\\
    \prob_{(x,x')\sim \gamma_\nu} (d(x,x')>\epsilon) = 0.
\end{align*}
Let $\gamma'\in \Pi(\mu', \nu')$ be the coupling that achieves the optimal transport cost $D_{TV}(\mu', \nu')$.
Construct a coupling $\gamma_0 \in \Pi(\mu, \nu)$ as $\gamma_0 = \gamma_\mu\circ \gamma'\circ\gamma_\nu$. Then,
\begin{align*}
    D_\epsilon(\mu, \nu) 
    &\leq \int_{\cX^2} \1\{d(x,x')>2\epsilon\} d\gamma_0\\
    &\leq \int_{\cX^2} \1\{d(x,x')>0\} d\gamma'\\
    &= D_{TV}(\mu', \nu').
\end{align*}
Since the above inequalty is true for any $\mu'$ and $\nu'$ such that $W_\infty(\mu, \mu')\leq \epsilon$ and $W_\infty(\nu, \nu')\leq \epsilon$, we have the following inequality.
\begin{align*}
    D_\epsilon(\mu, \nu)  
    \leq \inf_{W_\infty(p_0, p_0'), W_\infty(p_1, p_1')\leq \epsilon}
    D_{TV}(\mu', \nu').
\end{align*}

Now we will show the above inequality in the reverse direction.
Let $\gamma\in \Pi(\mu, \nu)$ be the coupling that achieves the optimal transport cost for $D_\epsilon(\mu, \nu)$.
Let $M: \cX^2\to \cX$ be a measurable midpoint map. (See \cite{Doh20} for why such a map exists.) That is, for all $(x,x')\in \cX^2$ we have
\begin{align*}
    d(x, M(x,x')) = d(x', M(x,x')) = \frac{1}{2} d(x,x').
\end{align*}
Consider a transport map $T: \cX^2\to\cX^2$ defined as
\begin{align*}
    T(x,x') = 
    \begin{cases}
    (M(x,x'), M(x,x')) & d(x,x')\leq 2\epsilon,\\
    (x,x') & \text{otherwise}.
    \end{cases}
\end{align*}
$T$ is measurable because it is piece-wise measurable on measurable sets. Further, it follows from the definition of $M$ that each coordinate of a point $(x,x')$ is  transported by $T$ by a distance no further than $\epsilon$. Let $\mu_0$ ad $\nu_0$ be the probability measures corresponding to the first and second marginals of $T_{\sharp \gamma}$ respectively. Then, $W_\infty(\mu, \mu_0)\leq \epsilon$ and $W_\infty(\nu, \nu_0)\leq \epsilon$. Hence,
\begin{align*}
    D_\epsilon(\mu, \nu) 
    &= \int_{\cX^2} \1\{d(x,x')>2\epsilon\} d\gamma\\
    &= \int_{\cX^2} \1\{d(x,x')>0\} d\gamma_{\sharp T}\\
    &\geq D_{TV}(\mu_0, \nu_0)\\
    &\geq \inf_{W_\infty(p_0, p_0'), W_\infty(p_1, p_1')\leq \epsilon}
    D_{TV}(\mu', \nu').
\end{align*}
Combining with the reverse inequality that we proved above, it is clear that the infimum over $D_{TV}$ is attained by $\mu_0$ and $\nu_0$.
\end{proof}

\begin{proof}[Proof of Lemma~\ref{lem: layered balls}]
The first equality in \eqref{eq: layered balls} follows from Theorem~\ref{thm: D_ep as shortest D_TV}. For the second equality, we have the following.
\begin{align*}
    \inf_{\substack{q\in \overline{\cP}(\cX):\\ q\preceq Tp_0}} D_\epsilon(q, p_1)
    &\stackrel{(i)}{=} 1 - (T+1) \inf_{A\in \cB(\cX)}  R_{\oplus\epsilon}(\ell_{0/1}, A)\\  
    &\stackrel{(ii)}{=} 1 - (T+1) \inf_{A\in \cB(\cX)} R_{\Gamma_\epsilon}(\ell_{0/1}, A)\\
    &=
    1 - (T+1) \inf_{A\in \cB(\cX)} \sup_{\substack{W_\infty(p_0, p_0')\leq \epsilon\\W_\infty(p_1, p_1')\leq \epsilon}}  r(A, p_0', p_1')\\    
    &\stackrel{(iii)}{\leq}
    1 - (T+1) \sup_{\substack{W_\infty(p_0, p_0')\leq \epsilon\\W_\infty(p_1, p_1')\leq \epsilon}} \inf_{A\in \cB(\cX)}   r(A, p_0', p_1')\\  
    &=
    \inf_{\substack{W_\infty(p_0, p_0')\leq \epsilon\\W_\infty(p_1, p_1')\leq \epsilon}} [1 - (T+1)  \inf_{A\in \cB(\cX)}   r(A, p_0', p_1')]\\    
    &\stackrel{(iv)}{\leq}
    \inf_{\substack{W_\infty(p_0, p_0')\leq \epsilon\\W_\infty(p_1, p_1')\leq \epsilon}} 
    \inf_{\substack{q'\in \overline{\cP}(\cX):\\ q'\preceq Tp_0'}}
    D_{TV}(q', p_1'),
\end{align*}
where $(i)$ follows from Theorem~\ref{thm: adv_risk as D_ep}, $(ii)$  from Theorem~\ref{thm: Adv risk in Polish Space}, $(iii)$   from Lemma~\ref{lem: max-min inequality}, and $(iv)$ again from Theorem~\ref{thm: adv_risk as D_ep} with $\epsilon = 0$.

We will now show the inequality in the opposite direction. That is, we will show the following.
\begin{align}
    \inf_{\substack{q\in \overline{\cP}(\cX):\\ q\preceq Tp_0}}
    \inf_{\substack{W_\infty(q, q')\leq \epsilon\\W_\infty(p_1, p_1')\leq \epsilon}} 
    D_{TV}(q', p_1')
    \geq 
    \inf_{\substack{W_\infty(p_0, p_0')\leq \epsilon\\W_\infty(p_1, p_1')\leq \epsilon}} 
    \inf_{\substack{q'\in \overline{\cP}(\cX):\\ q'\preceq Tp_0'}}
    D_{TV}(q', p_1')
\end{align}
Consider arbitrary probability measures $q', p_1'\in \overline{\cP}(\cX)$ generated in accordance with the constraints over the infimum terms on the left hand side of the above inequality. That is, let $q'$ and $p_1'$ be such that $W_\infty(q, q')\leq \epsilon$ and $W_\infty(p_1, p_1')\leq \epsilon$ where $q\preceq Tp_0$. We will now construct $p_0'\in \overline{\cP}(\cX)$ such that $q'\preceq Tp_0'$ and $W_\infty(p_0, p_0')\leq \epsilon$. This will show that the set of $q', p_1'\in \overline{\cP}(\cX)$ satisfying the constraints over the infimum terms on the right hand side is a superset of the corresponding set on the right hand side, and hence prove the above inequality.

Define a probability measure $p_0'\in \overline{\cP}(\cX)$ as $p_0'(A) = p_0(A) + \frac{1}{T}q'(A) - \frac{1}{T}q(A)$ for $A\in \cB(\cX)$. To show that $p_0'$ is a valid probability measure, we have the following.
\begin{align*}
    p_0'(\cX) &= p_0(\cX) + \frac{1}{T}q'(\cX) - \frac{1}{T}q(\cX) = 1\\
    p_0'(A) &= \frac{1}{T}(Tp_0(A) - q(A)) + \frac{1}{T}q'(A) \geq  \frac{1}{T}q'(A) \geq 0.
\end{align*}
The above equality also shows that $q'\preceq Tp_0'$. We will now show that $W_\infty(p_0, p_0')\leq \epsilon$. Since $ W_\infty(q, q')\leq\epsilon$, there exists $\gamma\in \Pi(q, q')$ such that $\gamma(\{(x,x')\in \cX^2: d(x,x')\leq 2\epsilon\}) = 1$. Define $\gamma' \in \Pi(p_0, p_0')$ as follows for $A\in \cB(\cX^2)$.
\begin{align*}
	\gamma'(A) = p_0(\{x\in \cX: (x,x)\in A\}) + \frac{1}{T} \gamma(A) - \frac{1}{T}q(\{x\in \cX: (x,x)\in A\}).
\end{align*}
To see that $\gamma' \in \Pi(p_0, p_0')$, we have the following for $A_1, A_2\in \cB(\cX)$.
\begin{align*}
    \gamma'(A_1\times \cX) 
    &= p_0(A_1) + \frac{1}{T}q(A_1) - \frac{1}{T}q(A_1) = p_0(A_1),\\
    \gamma'(\cX\times A_2)
    &= p_0(A_2) + \frac{1}{T}q'(A_2) - \frac{1}{T}q(A_2) = p_0'(A_2).
\end{align*}
Moreover,
\begin{align*}
	\gamma'(\{(x,x')\in \cX^2: d(x,x')\leq 2\epsilon\})
	= p_0(\cX) + \frac{1}{T}\gamma(\{(x,x')\in \cX^2: d(x,x')\leq 2\epsilon\}) - \frac{1}{T}q(\cX) = 1.
\end{align*}
Therefore, $W_\infty(p_0, p_0')\leq \epsilon$.
\end{proof}

\end{document}